\def\year{2020}\relax
\documentclass[letterpaper]{article} 
\usepackage{aaai20}  
\usepackage{times}  
\usepackage{helvet} 
\usepackage{courier}  
\usepackage[hyphens]{url}  
\usepackage{graphicx} 
\usepackage{graphicx}  
\usepackage{multirow}
\newcommand{\citet}[1]{\citeauthor{#1} \shortcite{#1}}
\newcommand{\citep}{\cite}

\usepackage{times}
\usepackage{courier}
\usepackage{graphicx}
\usepackage{algorithm}
\usepackage{algorithmic}

\usepackage{graphicx}

\usepackage{bbm}
\usepackage[utf8]{inputenc}
\usepackage[english]{babel}
\usepackage{amsthm}
\usepackage{amsmath}
\usepackage{amsfonts}
\usepackage{booktabs}
\usepackage[table,x11names]{xcolor}
\usepackage{subcaption}
\usepackage{changepage}
\usepackage{framed}

\newtheorem{theorem}{Theorem}
\newtheorem*{theorem*}{Theorem}
\newtheorem{proposition}{Proposition}
\newtheorem*{proposition*}{Proposition}
\newtheorem{lemma}{Lemma}

\newtheorem{definition}{Definition}
\newtheorem{assumption}{Assumption}

\newcommand{\h}{\mathcal{H}}
\newcommand{\E}{\mathbb{E}}

\newcommand{\R}{\mathbb{R}}

\newcommand{\A}{\mathcal{A}}
\newcommand{\T}{\mathcal{T}}

\newcommand{\F}{\mathcal{F}}
\newcommand{\Z}{\mathcal{Z}}
\newcommand{\U}{\mathcal{U}}

\newcommand{\pth}[1]{\left( #1 \right) }

\newcommand{\abs}[1]{{\left| #1 \right| }}
\newcommand{\braces}[1]{\left\{ #1 \right\} }
\newcommand{\cmnt}[1]{\ignorespaces}

\newcommand{\pearldo}[1]{\mathbf{do}\pth{#1}}

\newcommand{\pie}[2]{\pi_e^{(#1)}(#2)}
\newcommand{\pib}[2]{\pi_b^{(#1)}(#2)}

\newcommand{\todo}[1]{\textcolor{red}{\{TODO: #1\}}}

\newcommand{\uri}[1]{\textcolor{green}{\{Uri: #1\}}}

\DeclareMathSizes{10}{8}{6}{5}
\makeatletter
\def\Let@{\def\\{\notag\math@cr}}

\title{Off-Policy Evaluation in 
Partially Observable Environments}

\author{Guy Tennenholtz \\ Technion Institute of Technology
\And Shie Mannor \\ Technion Institute of Technology
\And Uri Shalit \\ Technion Institute of Technology}

\begin{document}
%

\maketitle
\begin{abstract}
This work studies the problem of batch off-policy evaluation for Reinforcement Learning in partially observable environments. Off-policy evaluation under partial observability is inherently prone to bias, with risk of arbitrarily large errors. We define the problem of off-policy evaluation for Partially Observable Markov Decision Processes (POMDPs) and establish what we believe is the first off-policy evaluation result for POMDPs. In addition, we formulate a model in which observed and unobserved variables are decoupled into two dynamic processes, called a Decoupled POMDP. We show how off-policy evaluation can be performed under this new model, mitigating estimation errors inherent to general POMDPs. We demonstrate the pitfalls of off-policy evaluation in POMDPs using a well-known off-policy method, Importance Sampling, and compare it with our result on synthetic medical data. 
\end{abstract}

\section{Introduction}

Reinforcement Learning (RL) algorithms learn to maximize rewards by analyzing past experience in an unknown environment \citep{sutton1998reinforcement}. In the context of RL, off-policy evaluation (OPE) refers to the task of estimating the value of an \textit{evaluation policy} without applying it, using data collected under a different \textit{behavior policy} \citep{dann2014policy}, also unknown as a logging policy.

The problem of OPE has been thoroughly studied under fully-observable models.  In this paper we extend and define OPE for Partially Observable Markov Decision Processes (POMDPs). Informally, the goal of OPE in POMDPs is to evaluate the cumulative reward of an \textit{evaluation} policy $\pi_e$ which is a function of \textit{observed} histories, using a measure over the observed variables under a \textit{behavior} policy $\pi_b$ which is a function of an \textit{unobserved} state. We assume that we do not have access to the unobserved states, nor do we have any prior information of their model. In fact, in many cases we do not even know whether these states exist. These states are commonly referred to as confounding variables in the causal inference literature, whenever they affect both the reward as well as the behavior policy. OPE for POMDPs is highly relevant to real-world applications in fields such as healthcare, where we are trying to learn from observed policies enacted by medical experts, without having full access to the information the experts have in hand.

A basic observation we make is that \textbf{traditional  methods in OPE are not applicable to partially observable environments}. For this reason, we start by defining the OPE problem for POMDPs, proposing various OPE approaches. 
We define OPE for POMDPs in Section~\ref{sec: preliminaries}. Then, in Section~\ref{sec: pomdp evaluation}, Theorem~\ref{thm: main result} shows how past and future observations of an unobserved state can be leveraged in order to evaluate a policy, under a non-singularity condition of certain joint probability distribution matrices. To the best of our knowledge, this is the first OPE result for POMDPs.

In Section~\ref{sec: Decoupled POMDP} we build upon the results of Section~\ref{sec: pomdp evaluation} and propose a more involved POMDP model: the Decoupled POMDP model. This is a class of POMDPs for which observed and unobserved variables are distinctly partitioned. Decoupled POMDPs hold intrinsic advantages over POMDPs for OPE. The assumptions required for OPE are more flexible than those required in POMDPs, allowing for easier estimation (Theorem~\ref{thm: main result2}). 

In Section~\ref{sec: importance sampling} we attempt to answer the question as to why traditional OPE methods fail when parts of the state are unobserved. We emphasize the hardness of OPE in POMDPs through a conventional procedure known as Importance Sampling. We further construct an Importance Sampling variant that can be applied to POMDPs under an assumption about the reward structure. We then compare this variant to the OPE result of Section~\ref{sec: Decoupled POMDP} in a synthetic medical environment (Section~\ref{sec: experiments}), showing it is prone to arbitrarily large bias.

Before diving into the subtleties associated with off-policy evaluation in partially observable environments, we provide two examples in which ignoring unobserved variables can lead to erroneous conclusions about the relation between actions and their rewards. 

\noindent \textbf{Medical Treatment:} Consider a physician  monitoring a patient, frequently prescribing drugs and applying treatments according to her medical state. Some of the patient's information observed by the physician may be unavailable to us (e.g., the patient's socioeconomic status). A physician might tend to prescribe Drug A for her wealthier patients who can afford it. At the same time, wealthier patients tend to have better outcomes regardless of the specific drug. As we are unaware of the doctor's inner state for choosing this action, and have no access to the information she is using, a naive model would wrongly deduce that prescribing drug A is more effective than it actually is.

\noindent \textbf{Autonomous driving:} Consider teaching an autonomous vehicle to drive using video footage of cameras located over intersections. In this scenario, many unobserved variables may be present, including: objects unseen by the camera, the driver's current mood, social cues between drivers and pedestrians, and so on. Naive estimations of policies based on other drivers' behavior may result in catastrophic outcomes. For the purpose of this illustration, let us assume tired drivers tend to be involved in more traffic accidents than non-tired drivers. In addition, suppose non-tired drivers tend to drive faster than tired drivers. We wish to construct a safe autonomous car based on traffic camera footage. Since the tiredness of the driver is unobserved, a naive model might wrongly evaluate a good policy as one that drives fast.

Understanding how to evaluate the effects of actions in the presence of unobserved variables that affect both actions and rewards is the premise of a vast array of work in the field of causal inference. Our present work  owes much to ideas presented in the causal inference literature under the name of effect restoration \citep{kuroki2014measurement}. In our work, we build upon a technique introduced by \citet{miao2018identifying} on causal inference using proxy variables. In the two papers above the unobserved variable is static, and there is only one action taken. Our work deals with dynamic environments with sequential actions and action -- hidden-state feedback loops. Surprisingly, while this is in general a harder problem, we show that these dynamics can be leveraged to provide us with multiple noisy views of the unobserved state.

Our work sits at an intersection between the fields of RL and Causal Inference. While we have chosen to use terminology common to RL, this paper could have equivalently been written in causal inference terminology. We believe it is essential to bridge the gap between these two fields, and include an interpretation of our results using causal inference terminology in the appendix.
\cmnt{
\begin{figure}[t!]
    \begin{center}
    \begin{subfigure}{0.515\linewidth}
    \includegraphics[width=\linewidth]{figures/sufficient_history.png}
    \caption{SS-POMDP}
    \end{subfigure}
    \hspace{0.25in}
    \begin{subfigure}{0.388\linewidth}
    \includegraphics[width=\linewidth]{figures/pomdp.png}
    \caption{POMDP}
    \end{subfigure}
    \hspace{0.3in}
    \caption{ Causal graphs of a SS-POMDP and POMDP. The blue line marks the behavior policy (i.e., the data generating process), whereas the red line marks the wanted evaluation. In all graphs, $\pi_e$ is marked as only a function of the current observation for brevity. Nevertheless, it may also be a function of its previous observed history. \uri{It's a bit confusing in the SS-POMDP case having both the small $r$'s and the big $R$}}
    \label{fig: causal diagram}
    \end{center}
\end{figure}
}

\section{Preliminaries}
\label{sec: preliminaries}
\subsection{Partially Observable Markov Decision Process}

We consider a finite-horizon discounted Partially Observable Markov Decision Process (POMDP). A POMDP is defined as the 5-tuple $(\U, \mathcal{A}, \Z, \mathcal{P}, O, r,\gamma)$ \cite{puterman1994markov}, where $\U$ is a finite state space, $\mathcal{A}$ is a finite action space, $\Z$ is a finite observation space, ${P : \U \times \U \times \mathcal{A} \mapsto [0,1]}$ is a transition kernel, ${O: \U \times \Z \mapsto [0,1]}$ is the observation function, where $O(u,z)$ denotes the probability $P(z|u)$ of perceiving observation $z$ when arriving in state $u$, ${r : \U \times \mathcal{A} \to [0,1]}$ is a reward function, and $\gamma\in(0,1)$ is the discount factor. A diagram of the causal structure of POMDPs is depicted in Figure~\ref{fig: Decoupled POMDP}a.

A POMDP assumes that at any time step the environment is in a state $u \in \U$, an agent takes an action $a \in \A$ and receives a reward $r(u, a)$ from the environment as a result of this action. At any time $t$, the agent will have chosen actions and received rewards for each of the $t$ time steps prior to the current one. The agent's observable history at time $t$ is defined by ${h_t^o = (z_0, a_0, \hdots, z_{t-1}, a_{t-1}, z_t)}$. We denote the space of observable histories at time $t$ by $\h_t^o$.

We consider trajectories and observable trajectories. A trajectory of length $t$ is defined by the sequence $\tau = (u_0, z_0, a_0, \hdots, u_t, z_t, a_t)$. Similarly, an observable trajectory of length $t$ is defined by the sequence ${\tau^o = (z_0, a_0, \hdots, z_{t-1}, a_{t-1}, z_t, a_t)}$. We denote the space of trajectories and observable trajectories of length $t$ by $\T_t$ and $\T_t^o$, respectively. Finally, a policy $\pi$ is any stochastic, time-dependent\footnote{For brevity we sometimes denote policies by $\pi$, though they may depend on time such that $\pi = \pi(t)$.} mapping from a measurable set $\mathcal{X}_t \subset \T_t$ to the set of probability measures on the Borel sets of $\mathcal{A}$, denoted \mbox{by $\mathcal{B}(\A)$.}

For any time $t$, and trajectory $\tau \in \T_t$, we define the cumulative reward
\begin{equation*}
R_t(\tau) = \sum_{k=0}^t \gamma^k r(u_k, a_k).
\end{equation*}
The above is also known as the discounted return. Given any policy $\pi$ and initial distribution over states $\U$, denoted by $\nu_0$, we define the expected discounted return at time $L$ by
\begin{equation*}
v_L(\pi ; \nu_0) = \E \pth{R_L(\tau) \middle| u_0 \sim \nu_0, \tau \sim \pi}.
\end{equation*}
When clear from context, we will assume $\nu_0$ and $L$ are known and fixed and simply write $v(\pi)$.

\subsection{Policy Evaluation}

Off-policy evaluation (OPE) considers two types of policies: a behavior policy and an evaluation policy, as defined below.
\begin{definition}[behavior and evaluation policies]~\\
\label{def: policy types}
A \textit{behavior} policy, denoted by $\pi_b^{(t)}$, is a stochastic, time-dependent mapping from states $\U$ to $\mathcal{B}(\A)$. \footnote{We consider behavior policies to be functions of unobserved states. This assumption is common in MDPs in which the state is observed, as it is known that there exists a stationary optimal policy that is optimal \citep{romanovskii1965existence}. In addition, the unobserved state is assumed to contain all required information for an agent to make an optimal decision. } \\
An \textit{evaluation} policy, denoted by $\pi_e^{(t)}$, is a stochastic, time-dependent mapping from observable histories $\h_t^o$ to $\mathcal{B}(\A)$.
\end{definition}

For any time step $t$ and policy $\pi$ let $P^\pi(\cdot)$ be the measure over observable trajectories $\h_t^o$, induced by policy $\pi$. We will denote this measure by $P^b, P^e$, whenever $\pi$ is a behavior or evaluation policy, respectively. We are now ready to define off-policy evaluation in POMDPs:
\\
\begin{adjustwidth}{20.5pt}{20.5pt}
\textit{The goal of off-policy evaluation in POMDPs is to evaluate $v_L(\pi_e)$ using the measure $P^b(\cdot)$ over observable trajectories $\T_L^o$ and the given policy $\pi_e$.}
\end{adjustwidth}
~\\
This corresponds to the scenario in which data comes from a system which records an agent taking actions based on her own information sources ($u$), and we want to evaluate a policy $\pi_e$ which we learn based only on information available to the learning system ($\tau^o$).

\subsection{Vector Notations}

Let $x, y, z$ be random variables accepting values in $\braces{x_1, \hdots, x_{n_1}}, \braces{y_1, \hdots, y_{n_2}}, \braces{z_1, \hdots, z_{n_3}}$, respectively, and let $\F$ be a filtration that includes all information: states, observations, actions, and rewards. We denote by $P(X | y, \F)$ the $n_1 \times 1$ column vector with elements ${\pth{P(X | y, \F)}_i = P(x_i | y, \F)}$. Similarly we denote by $P(x | Y, \F)$ the $1 \times n_2$ row vector with elements ${\pth{P(x | Y, \F}_i = P(x | y_i, \F)}$. Note that ${P(x | Y, \F)P(Y | \F) = P(x | \F)}$. Finally, let ${P(X | Y, \F)}$ be the ${n_1 \times n_2}$ matrix with elements ${\pth{P(X | Y, \F)}_{ij} = P(x_i| y_j, \F)}$. Note that if $x$ is independent of $z$  given $y$ then we have the matrix equality ${P(X | Y, \F)P(Y | Z, \F) = P(X | Z, \F)}$.

We will sometimes only consider subsets of the above matrices. For any index sets ${I \subset \{1, \hdots n_1\}, J \subset \{1, \hdots n_2\}}$ let $P_{(I, J)}(X | Y, \F)$ be the matrix with elements ${\pth{P_{(I, J)}(X | Y, \F)}_{ij} = P(x_{I_i}| y_{J_j}, \F)}$.

\begin{figure*}[t!]
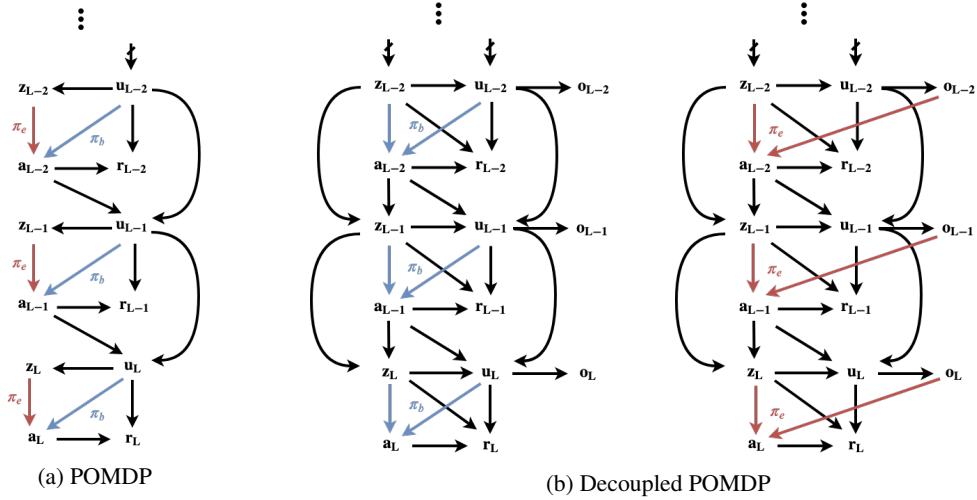

    \begin{center}
    \begin{subfigure}{0.16\linewidth}
    \includegraphics[width=\linewidth]{figures/pomdp.png}
    \caption{POMDP}
    \end{subfigure}
    \hspace{0.5in}
    \begin{subfigure}{0.52\linewidth}
    \includegraphics[width=0.475\linewidth]{figures/depomdp_pib.png}%
    \hfill
    \includegraphics[width=0.475\linewidth]{figures/depomdp_pie.png}
    \caption{Decoupled POMDP}
    \end{subfigure}
    \caption{A causal diagram of a POMDP (a) and a Decoupled POMDP (b). In Decoupled POMDPs, observed and unobserved states are separated into two distinct processes, with a coupling between them at each time step. Diagrams depicts the causal dependence of a behavior policy and evaluation policies. While evaluation policies are depicted to depend on the current observation alone, they can depend on any observable history $h_t^o$. }
    \label{fig: Decoupled POMDP}
    \end{center}
\end{figure*}

\section{Policy Evaluation for POMDPs}
\label{sec: pomdp evaluation}

In this section we show how past and future observations can be leveraged in order to create an unbiased evaluation of $\pi_e$ under specific invertibility conditions. 
It is a generalization of the bandit-type result presented in \citet{miao2018identifying}, where causal effects are inferred in the presence of an unobserved discrete confounder, provided one has two conditionally independent views of the confounder which are non-degenerate (i.e., the conditional probability matrices are invertible). We show how time dynamics readily give us these two conditionally independent views - using the past and future observations as two views of the unobserved state. 

\noindent For any ${\tau^o = (z_0, a_0, \hdots, z_t, a_t) \in \T^o_t}$ we define the generalized weight matrices
\begin{equation*}
    W_i(\tau^o) =
    P^b(Z_i| a_i, Z_{i-1})^{-1}
    P^b(Z_i, z_{i-1} | a_{i-1}, Z_{i-2})
\end{equation*}
for $i \geq 1$, and
\begin{equation*}
    W_0(\tau^o)
    =
    P^b(Z_0| a_0, Z_{-1})^{-1}P^b(Z_0).
\end{equation*}
 Here, we assume there exists an observation of some time step before initial evaluation (i.e., $t < 0$), which we denote by $z_{-1}$. Alternatively, $z_{-1}$ may be an additional observation that is independent of $z_0$ and $a_0$ given $u_0$. Note that the matrices $W_i$ can be estimated from the observed trajectories of the behavior distribution. We then have the following result.
\begin{theorem}[POMDP Evaluation]
\label{thm: main result}
Assume $\abs{\Z} \geq \abs{\U}$ and that $P^b(Z_i | a_i, Z_{i-1})$ are invertible for all $i$ and all $a_i \in \A$.
For any $\tau^o \in \T^o_t$ denote
 \begin{align*}
     &\Pi_e(\tau^o) = \prod_{i=0}^t \pie{i}{a_i | h_i^o}, ~~
     &\Omega(\tau^o) = \prod_{i=0}^{t} W_{t-i}(\tau^o).
 \end{align*}
Then
\begin{equation*}
    P^e(r_t) 
    =
    \sum_{\tau^o \in \T^o_t}
    \Pi_e(\tau^o)
    P^b(r_t, z_t | a_t, Z_{t-1})
    \Omega(\tau^o).
\end{equation*}
\end{theorem}
\begin{proof}
See Appendix.
\end{proof}

Having evaluated $P^e(r_t)$ for all $0 \leq t \leq L$ suffices in order to evaluate $v(\pi_e)$. Theorem~\ref{thm: main result} lets us evaluate $v(\pi_e)$ without access to the unknown states $u_i$. It uses past and future observations $z_t$ and $z_{t-1}$ as proxies of the unknown state. Its main assumptions are that the conditional distribution matrices $P^b(Z_i | a_i, Z_{i-1})$ and $P^b(U_i | a_i, Z_{i-1})$ are invertible. In other words, it is assumed that enough information is transferred from states to observations between time steps. Consider for example the case where $U_i$ and $Z_i$ are both binary, then a sufficient condition for invertibility to hold is to have $p(z_i=1|z_{i-1}=1,a_i) \neq p(z_i=1|z_{i-1}=0,a_i)$ for all values of $a_i$.
A trivial case in which this assumption does not hold is when $\{u_i\}_{0 \leq i \leq L}$ are i.i.d. In such a scenario, the observations $z_i, z_{i-1}$ do not contain enough useful information of the unobserved state $u_i$, and additional independent observations of $u_i$ are needed. In the next section we will show how this assumption can be greatly relaxed under a decoupled POMDP model. Particularly, in the next section we will devise an alternative model for partial observability, and analyze its superiority over POMDPs in the context of OPE. 


\section{Decoupled POMDPs}
\label{sec: Decoupled POMDP}
Theorem~\ref{thm: main result} provides an exact evaluation of $\pi_e$. However it assumes non-singularity of several large stochastic matrices. While such random matrices are likely to be invertible (see e.g., \citet{bordenave2012circular} Thm~1.4), estimating their inverses from behavior data can lead to large approximation errors or require very large sample sizes. This is due to the  structure of POMDPs, which is confined to a causal structure in which unobserved states form observations. This restriction is present even when $O(u, z)$ is a deterministic measure. In many settings, one may detach the effect of unobserved and observed variables. In this section, we define a Decoupled Partially Observable Markov Decision Process (Decoupled POMDP), in which the state space is decoupled into unobserved and observed variables. 

Informally, in Decoupled POMDPs the state space is factored into observed and unobserved states. Both the observed and unobserved states follow Markov transitions. In addition, unobserved states emit independent observations. As we will show, Decoupled POMDPs are more appropriate for OPE under partial observability. Contrary to Theorem~\ref{thm: main result}, Decoupled POMDPs use matrices that scale with the support of the unobserved variables, which, as they are decoupled from observed variables, are of much smaller cardinality. 

\begin{definition}
We define a finite-horizon discounted Decoupled Partially Observable Markov Decision Process (Decoupled POMDP) as the tuple $(\U, \Z, \mathcal{O}, \mathcal{A}, \mathcal{P}, \mathcal{P}_{\mathcal{O}}, r, \gamma)$, where ${\Z}$ and ${\U}$ consist of an observed and unobserved finite state space, respectively. $\mathcal{A}$ is the action space, $\mathcal{O}$ is the independent observation space, ${\mathcal{P} : \Z \times \U \times \Z \times \U \times \A  \mapsto [0,1]}$ is the transition kernel, where $\mathcal{P}(z', u'|z, u, a)$ is the probability of transitioning to state $(z', u')$ when visiting state $(z, u)$ and taking action $a$, ${\mathcal{P}_\mathcal{O}: \U \times \mathcal{O} \mapsto [0,1]}$ is the independent observation function, where $\mathcal{P}_\mathcal{O}(o | u)$ is the probability of receiving observation $o$ when arrive at state $u$, ${r : \U \times \Z \times \mathcal{A} \to [0,1]}$ is a reward function, and $\gamma\in(0,1)$ is the discount factor.
\end{definition}

A Decoupled POMDP assumes that at any time step the environment is in a state $(u, z) \in \U \times \Z$, an agent takes an action $a \in \A$ and receives a reward $r(u, z, a)$ from the environment. The agent's observable history is defined by ${h_t^o = (z_0, o_0, a_0, \hdots, z_{t-1}, o_{t-1}, a_{t-1}, z_t, o_t)}$. With abuse of notations we denote the space of observable histories at time $t$ by $\h_t^o$. We similarly define trajectories and observable trajectories of length $t$ by $\tau = (u_0, z_0, o_0, a_0, \hdots, u_t, z_t, o_t, a_t)$ and ${\tau^o = (z_0, o_0, a_0, \hdots, z_t, o_t, a_t)}$, respectively. With abuse of notations, we denote the space of trajectories and observable trajectories of length $t$ by $\T_t$ and $\T_t^o$, respectively. 
In Figure~\ref{fig: Decoupled POMDP}(b) we give an example of a Decoupled POMDP (used in Theorem \ref{thm: main result2}) for which $z_i$ causes $u_i$. 

Decoupled POMDPs hold the same expressive power and generality of POMDPs. To see this, one may remove the observed state space $\Z$ to recover the original POMDP model. Nevertheless, as we will see in Theorem~\ref{thm: main result2}, Decoupled POMDPs also let us leverage their structure in order to achieve tighter results. Similar to POMDPs, OPE for Decoupled POMDPs considers behavior and evaluation policies.  
\begin{definition}[behavior and evaluation policies]~\\
A \textit{behavior} policy, denoted by $\pi_b^{(t)}$, is a stochastic, time-dependent mapping from states $\U \times \Z$ to $\mathcal{B}(\A)$. \\
An \textit{evaluation} policy, denoted by $\pi_e^{(t)}$, is a stochastic, time-dependent mapping from observable histories $\h_t^o$ to $\mathcal{B}(\A)$.
\end{definition}

The goal of OPE for Decoupled POMDPs is defined similarly as general POMDPs. Decoupled POMDPs allow us to model environments in which observations are not contained in the unknown state. They are decoupled by a Markovian processes which governs both observed and unobserved variables. In what follows, we will show how this property can be leveraged for evaluating a desired policy.

\subsection{Policy Evaluation for Decoupled POMDPs}

For all $i$, let $K_i \subset \{1, \hdots, \abs{\mathcal{O}}\}, J_i \subset \{1, \hdots, \abs{\mathcal{Z}}\}$ such that $\abs{K_i} = \abs{J_i} = \abs{U}$. Similar to before, for any ${\tau^o \in \T_t^o}$, we define the generalized weight matrices
\begin{align*}
    G_i(\tau^o)
    &=
    P^b_{(K_i,J_{i-1})}(O_i| z_i, a_i, Z_{i-1})^{-1}
    \times \\
    &\times P^b_{(K_i,J_{i-2})}(O_i, o_{i-1}, z_i | z_{i-1}, a_{i-1}, Z_{i-2}),
\end{align*}
for $i \geq 1$ and
\begin{equation*}
    G_0(\tau^o)
    =
    P^b_{(K_0,J_{-1})}(O_0| z_0, a_0, Z_{-1})^{-1}
    P^b_{K_0}(O_0 | z_0)P^b(z_0).
\end{equation*}
We then have the following result.

\begin{theorem}[Decoupled POMDP Evaluation]
\label{thm: main result2}
Assume $\abs{\mathcal{Z}}, \abs{\mathcal{O}} \geq \abs{\U}$ and that there exist index sets $K_i, J_i$ such that ${\abs{K_i} = \abs{J_i} = \abs{\U}}$ and $P_{(K_i,J_{i-1})}^b(O_i | a_i, z_i, Z_{i-1})$ are invertible $\forall i, a_i, z_i \in \A \times \Z$. In addition assume that $z_{i-1}$ is independent of $u_{i+1}$ given $z_{i+1}, a_i, u_i$, $\forall i$ under $P^b$. \\
For any ${\tau^o \in \T_t^o}$, denote by
 \begin{align*}
     &\Pi_e(\tau^o) =
     \prod_{i=0}^t \pie{i}{a_i | h_i^o}, ~~
     &\Omega(\tau^o) = \prod_{i=0}^{t}
    G_{t-i}(\tau^o).
 \end{align*}
Then
\begin{equation*}
    P^\pi(r_t)
    =
    \sum_{\tau^o \in \T_t^o}
    \Pi_e(\tau^o)
    P^b_{K_{t-1}}(r_t, o_t| z_t, a_t, Z_{t-1})
    \Omega(\tau^o).
\end{equation*}
\end{theorem}
\begin{proof}
See Appendix.
\end{proof}

Approximating Theorem~\ref{thm: main result2}'s result under finite datasets is more robust than Theorem~\ref{thm: main result}. For one, the matrices are of size $\abs{\mathcal{U}}$, which can be much smaller than $\abs{\mathcal{Z}} + \abs{\mathcal{U}}$. In addition, and contrary to Theorem~\ref{thm: main result}, the result holds for any index set $J_i, K_i \subset [\abs{\mathcal{Z}}]$ of cardinality $\abs{\mathcal{U}}$.  We can thus choose any of $\abs{\mathcal{Z}} \choose \abs{\mathcal{U}}$ possible subsets from which to approximate the matrices $G_i(\tau)$. This enables us to choose indices for solutions with desired properties (e.g., small condition numbers of the matrices). We may also construct an estimator based on a majority vote of the separate estimators. Finally, we note that solving for $G_i(\tau)$ for any $J_i, K_i$ can be done using least squares regression. 

Up to this point we have shown how the task of OPE can be carried out in two settings: general POMDPs and Decoupled POMDPs. The results in Theorems~\ref{thm: main result} and~\ref{thm: main result2} depend on full trajectories, which we believe are a product of the high complexity inherent to the problem of OPE with unobservable states. In the next section, we demonstrate the hardness inherent to OPE in these settings through an alternative OPE method - a variant of an Importance Sampling method \citep{precup2000eligibility}. We then experiment and compare these different OPE techniques on a synthetic medical environment.

\section{Importance Sampling and its Limitations in Partially Observable Environments}
\label{sec: importance sampling}

In previous sections we presented OPE in partially observable environments, and provided, for what we believe is the first time, techniques of exact evaluation. A reader familiar with OPE in fully-observable environments might ask, why do new techniques need to be established and where do traditional methods fail? To answer this question, in this section we demonstrate the bias that may arise under the use of long-established OPE methods. More specifically, we demonstrate the use of a well-known approach, Importance Sampling (IS): a reweighting of rewards generated by the behavior policy, $\pi_b$, such that they are equivalent to unbiased rewards from an evaluation policy $\pi_e$.

Let us begin by defining the IS procedure for POMDPs. Suppose we are given a trajectory $\tau \in \T_t$. We can express $P^e(\tau)$ using $P^b(\tau)$ as
\begin{align*}
    &P^e(\tau) = P^e(u_0, z_0, a_0, \hdots ,u_t, z_t, a_t)\\
    &=
    \nu_0(u_0) \prod_{i=0}^{t-1}P^b(u_{i+1} | u_i, a_i) \prod_{i=0}^{t}P^b(z_{i} | u_{i})\pie{i}{a_i | h_i^o}\\
    &=
    P^b(\tau) \prod_{i=0}^{t} \frac{\pie{i}{a_i | h_i^o}}{\pib{i}{a_i | u_i}}.
\end{align*}
We refer to $w_i = \frac{\pie{i}{a_i | h_i^o}}{\pib{i}{a_i | u_i}}$ as the importance weights. The importance weights allow us to evaluate $v(\pi_e)$ using the data generating process $P^b(\cdot)$ as
\begin{align}
    \text{IS}(\pi_e, w) 
    :=
    \E\pth{R_L(\tau) \prod_{i=0}^{L}w_i \middle| \tau \sim \pi_b, u_0 \sim \nu_0 }.
    \label{eq: importance sampling}
\end{align}
  Note that $v(\pi_e) = \text{IS}(\pi_e, w)$. Unfortunately, the above requires the use of $\pib{i}{a_i | u_i}$, which are unknown and \textit{cannot} be estimated from data, as $u_i$ are unobserved under the POMDP model.

\begin{figure}[t!]
    \begin{center}
    \includegraphics[width=0.8\linewidth]{figures/pomdp_example2.png}
    \caption{ An example of a POMDP with 6 states and 2 observations for which importance sampling with importance weights $w_i = \frac{\pie{i}{a_i | h^o_i}}{P^b(a_{i} | h^o_i)}$ is biased. Numbers on arrows correspond to probabilities. Arrows marked by $a_0, a_1$ correspond to rewards or transitions of these actions. Rewards depend on values of $\alpha > 0$. Initial state distribution is $\nu_0 = (\frac{1}{2}, \frac{1}{2})$.}
    \label{fig: pomdp_example}
    \end{center}
\end{figure}

\begin{figure*}[ht!]
    \begin{center}
    \begin{subfigure}{0.3\linewidth}
    \includegraphics[width=\linewidth]{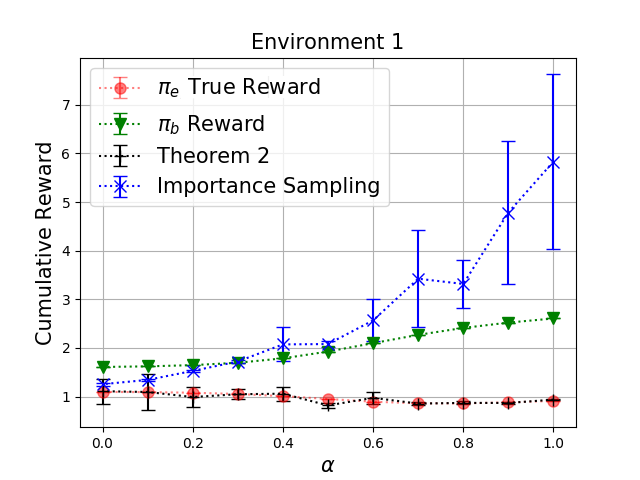}
    \end{subfigure}
    \begin{subfigure}{0.3\linewidth}
    \includegraphics[width=\linewidth]{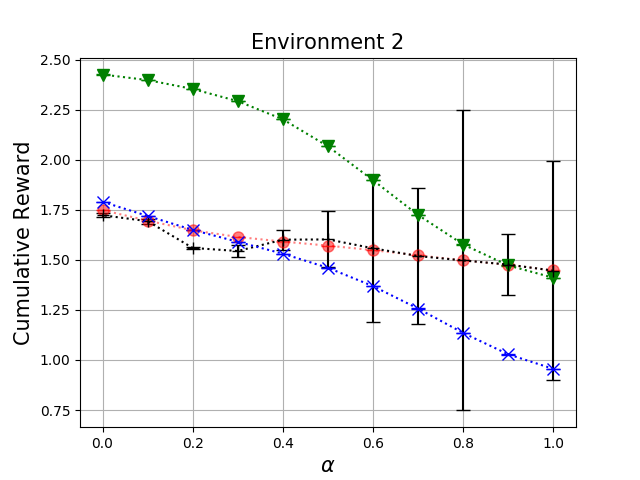}
    \end{subfigure}
    \begin{subfigure}{0.3\linewidth}
    \includegraphics[width=\linewidth]{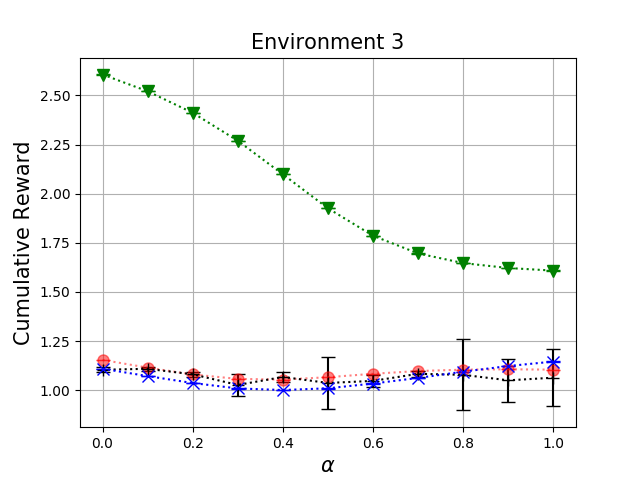}
    \end{subfigure}
    \caption{Comparison of cumulative reward approximation on three distinct synthetic environments. $\pi_b$ (green, triangles) and $\pi_e$ (red, circles) plots depict the true cumulative rewards of the behavior and evaluation policies, respectively. Ideally we would want the black ``Theorem 2'' curve and the blue IS curve to match the red curve of the true reward.}
    \label{fig: Experiments}
    \end{center}
\end{figure*}

\subsection{Sufficient Condition for Importance Sampling}

As Equation~\eqref{eq: importance sampling} does not resolve the off-policy problem, it remains an open question whether an IS procedure can be used in general POMDPs. Here, we give sufficient conditions for which a variant of IS can properly evaluate $\pi_e$. More specifically, we assume a POMDP which satisfies the following condition.
\begin{assumption}[Sufficient Condition for IS]
\begin{align*}
&\exists f_L: \T_L^o \mapsto \R \text{ s.t. } \forall \tau \in \T_L, R_L(\tau) = f_L(\tau^o) \text{ and} \\
&\forall 0 \leq i \leq L-1, \tau^o_i \in \T^o_i, z_{i+1} \in \Z, P^b(z_{i+1} | \tau^o_i) = P^e(z_{i+1} | \tau^o_i).
\end{align*}
\vspace{-0.5cm}
\label{assumption: sufficient condition for IS}
\end{assumption}
In other words, this assumption states that the observed trajectory at time $L$ is a sufficient statistic of the true reward, and $\tau_i^o$ is a sufficient statistic of the state $u_i$. Under Assumption~\ref{assumption: sufficient condition for IS} we can construct an IS variant as follows. Given a trajectory $\tau \in \T_L$ we have that
\begin{align*}
    v(\pi_e)
    &= \sum_{\tau \in \T_L} R(\tau)P^e(\tau)  \\
    &= \sum_{u_0, \hdots, u_L} \sum_{\tau^o \in \T^o_L} f_L(\tau^o)P^e(\tau^o, u_0, \hdots, u_L)  \\
    &= \sum_{\tau^o \in \T_L^o} f_L(\tau^o)P^e(\tau^o).
\end{align*}
We can express $P^e(\tau^o)$ using $P^b(\tau^o)$ for any $\tau^o \in \T^o_L$ as
\begin{align*}
    &P^e(\tau^o) = P^e(z_0, a_0, \hdots ,z_t, a_t)\\
    &= 
    P^b(z_0) \prod_{i=0}^{L-1} P^b(z_{i+1} | \tau^o_i)\prod_{i=0}^{L}\pie{i}{a_i | h^o_i} \\
    &= 
    P^b(\tau^o) 
    \prod_{i=0}^{L}
    \frac{\pie{i}{a_i | h^o_i}}{P^b(a_{i} | h^o_i)}.
\end{align*}
We can thus evaluate $v(\pi_e)$ using Equation~\eqref{eq: importance sampling} and importance weights $w_i = \frac{\pie{i}{a_i | h^o_i}}{P^b(a_{i} | h^o_i)}$. While this estimate seems simple and intuitive, as we demonstrate next, arbitrarily large evaluation errors can occur using the above IS weights when Assumption~\ref{assumption: sufficient condition for IS} does not hold.

\subsection{Importance Sampling Error in POMDPs}

In general POMDPs, if Assumption~\ref{assumption: sufficient condition for IS} does not hold, using the importance weights $w_i = \frac{\pie{i}{a_i | h^o_i}}{P^b(a_{i} | h^o_i)}$ can result in large errors in the evaluation of $\pi_e$. This error is demonstrated by an example given in Figure~\ref{fig: pomdp_example}. In this example, we assume an initial state distribution $\nu_0 = (\frac{1}{2}, \frac{1}{2})$. Given a behavior policy
${
    \pi_b(a_i | u_k^{(j)}) =
    \begin{cases}
    \frac{2}{3} &  i \oplus j = 0\\
    \frac{1}{3} & i \oplus j = 1 \\
    \end{cases}
}
$
for all $k$, and a stationary evaluation policy
${
    \pi_e(a_i | z_j) =
    \begin{cases}
    \frac{2}{3} & i \oplus j = 0 \\
    \frac{1}{3} & i \oplus j = 1
    \end{cases}
}$,
we have that ${v(\pi_b) \approx 0.72\alpha + 0.26\gamma}$ and ${v(\pi_e) \approx -0.01\alpha + 0.14\gamma}$. However, using ${w_i = \frac{\pi_e(a_i | z_i)}{P^b(a_i | h_i^o)}}$, IS evaluation yields ${\text{IS}(\pi_e, w) \approx 0.62\alpha + 0.34\gamma}$. This amounts to an error of $0.63\alpha + 0.2\gamma$ between the true policy evaluation and its importance sampling evaluation, which can be arbitrarily large. As an example, for $\alpha = 0.8\gamma$ we get that $v(\pi_b) = \text{IS}(\pi_e, w)$. Particularly, $\text{IS}(\pi_e, w) > v(\pi_b)$ for $\alpha > 0.8$. This is contrary to the fact that $v(\pi_b) > v(\pi_e)$ for all values of $\alpha > 0$.

Unlike the invertibility assumptions of Theorems~\ref{thm: main result} and~\ref{thm: main result2}, Assumption~\ref{assumption: sufficient condition for IS} is a strong assumption, that is unlikely to hold for almost any POMDP. At this point it is unclear if a more lenient assumption can assure an unbiased IS estimator for POMDPs, and we leave this as an open research question for future work. In the next section, we experiment with the results of Theorem~\ref{thm: main result2} and the IS variant constructed in this section on a finite-sample dataset generated by a synthetic medical environment.

\section{Experiments}
\label{sec: experiments}

The medical domain is known to be prone to many unobserved factors affecting both actions and rewards \citep{gottesman2018evaluating}. As such, OPE in these domains requires adaptation to partially observable settings. In our experiments we construct a synthetic medical environment using a Decoupled POMDP model, as described next. 

We denote $\sigma(x) = \frac{1}{1 + e^{-x}}$. The environment consists of a patient's (observed) medical state $z$. Here $z$ changes according to the taken action by $P(z' | z, a) \propto \sigma(c_{z,z',a}^\top \phi_z(z))$, where $c_{z, z', a}$ are given weights, and $\phi_z(z)$ are the state features. We assume there exist unobserved variables $u = (u_{\text{mood}}, u_{\text{look}})$ relating to the doctors current mood, and how ``good'' the patient looks to her, respectively. In other words, we assume the doctor has an inner subjective ranking of the patient's look. Observation of the doctor's mood and inner ranking are modeled by $P(o | u) \propto \sigma(c_{u, o}^\top \phi_{u}(u))$, where $c_{u,o}$ are given weights, and $\phi_{u}(u)$ are the unobserved state features. Such observations could be based on the doctor's textual notes. Such notes, when processed through state of the art sentiment analysis algorithms \citep{qian2016linguistically} act as proxies to the doctor's mood and subjective assessment of the patient.

We model the doctor's mood changes according to the doctor's current mood, patient's look, and taken action $P(u'_{\text{mood}} | u, a) \propto \sigma(c_{u, a, u'_\text{mood}}^\top \phi_{\text{mood}}(u))$. Finally, the doctor's inner ranking of the patient's look is dependent on the patient's state $z$ and look $u_{\text{look}}$ by ${P(u'_{\text{look}} | z',u_{\text{look}}) \propto \sigma(c_{z',u_{\text{look}}, u'_{\text{look}}}^\top (\phi_{\text{look}}(u_\text{look}), \phi_{z}(z)))}$.

We assume data generated by a confounded reward function and behavior policy
\begin{align*}
&r(u,z,a) \propto \sigma((1-\alpha)(c^r_{z, a})^\top \phi_z(z) + \alpha (c^r_{u, a})^\top \phi_u(u)), \\
&\pi_b(a | u, z) \propto \sigma((1-\alpha)(c^b_{z, a})^\top \phi_z(z) + \alpha (c^b_{u, a})^\top \phi_u(u)).
\end{align*}
Here ${\alpha \in [0, 1]}$ is a parameter which controls the ``level of confoundedness''. In other words, $\alpha$ is a measure of the intensity in which $\pi_b$ and $r$ depend on the unobserved state $u$, with $\alpha=0$ corresponding to no unobserved confounding. 

The spaces $\mathcal{Z}$, $\U$, and $\mathcal{O}$ were composed of two binary features each. We run the experiment in three environments, corresponding to different settings of the vectors $c$ meant to illustrate different behaviors of our methods. Ten million trajectories were sampled from the policy $\pi_b$ over a horizon \mbox{of 4} time steps for each environment. Figure~\ref{fig: Experiments}  depicts the cumulative reward of $\pi_{e}$, $\pi_b$, and their corresponding estimates according to Theorem~\ref{thm: main result2} and the IS weights $w_i^k = \frac{\pie{i}{a_i | h^o_i}}{P^b(a_{i} | h^o_i)}$, for different values of $\alpha$.

Environment 1 illustrates a typical result from the generative process above, where the vectors $c$ were sampled from a Gaussian. It is clear that IS estimates increase in bias with $\alpha$, whereas Theorem~\ref{thm: main result2}'s estimate remains unchanged. Moreover, for values of $\alpha > 0.3$, IS suggests that $\pi_e$ is superior to $\pi_b$. This implies that potentially arbitrarily bad policies could be learned by an off-policy RL algorithm. Environments~2 and 3 are atypical, and were found through deliberate search, to illustrate situations in which our estimation procedure does not clearly outperform IS. In Environment~2, for large values of $\alpha$, variance increases, due to near non-invertibility of the conditional probability matrices. In Environment~3, despite confounding, IS remains unbiased. 

\section{Related Work}

\noindent \textbf{POMDPs:} Uncertainty is a key feature in real world applications. POMDPs provide a principled general framework for planning under uncertainty \citep{spaan2012partially,williams2007partially}. POMDPs model aspects such as the stochastic effects of actions, incomplete information and noisy observations over the environment. POMDPs are known to be notoriously hard to approximate \citep{madani1999undecidability,papadimitriou1987complexity}. Their intractability is mainly due to the ``curse of dimensionality'' for which complexity grows exponentially with the cardinality of the unobserved state space. As this work considers offline evaluation under uncertainty, the unobserved state is treated as a confounding element for policy evaluation.

To the best of our knowledge our work provides the first OPE results for POMDPs. Nevertheless, there has been considerable work on learning in POMDPs, where, similar to our setting, the agent does not gain access to the POMDP model. \citet{even2005reinforcement} implement an approximate reset strategy based on a random walk of the agent, effectively resetting the agent's belief state. \citet{hausknecht2015deep} tackle the learning problem by adding recurrency to Deep Q-Learning, allowing the Q-network to estimate the underlying system state, narrowing the gap between $Q_\theta(o, a)$ and $Q_\theta(u, a)$. Nevertheless, work on learning in POMDPs greatly differs from OPE for POMDPs, as the former offer online solutions based on interactive environments, whereas the latter uses batch data generated by an \textit{unknown} and \textit{unregulated} behavior policy. 

\noindent \textbf{Off-Policy Evaluation (OPE):} Contemporary OPE methods can be partitioned into three classes: (1) direct methods (DM) \citep{precup2000eligibility,munos2016safe,le2019batch,jiang2015doubly}, which aim to fit the value of a policy directly, (2) inverse propensity score (IPS) methods, also known as importance sampling (IS) methods \citep{liu2018breaking,dudik2011doubly,jiang2015doubly}, and (3) Doubly-Robust methods (DRM) \citep{jiang2015doubly,thomas2016data,kallus2019double}, which combine IS methods with an estimate of the action-value function, typically supplied by a DM. These algorithms were designed for bandits, and later generalized to RL. Nevertheless, existing methods assume full observability of the underlying state. They become dubious when part of the data generating process is unobserved or unknown.

\noindent \textbf{Causal Inference:} A major focus of work in causal inference is how to estimate, in an offline model, the effects of actions without fully observing the covariates which lead to the action \citep{Pearl:2009:CMR:1642718,spirtes2000causation}. Much of the work in this field focuses on static settings, with some more recent work also tackling the bandit setting \citep{bareinboim2015bandits,forney2017counterfactual,ramoly2017causal,sen2016contextual}. Sufficient ``sequential ignorability'' conditions (no hidden confounding) and methods for OPE of causal effects under dynamic policies are given by \citep{murphy2001marginal,hernan2006comparison,hernan2019causality}.


Recently, there has been growing interest in handling unobserved confounders in the context of MDPs. \citet{zhang2016markov} consider a class of counterfactual policies that incorporate a notion they call ``intuition'', by using observed actions as input to an RL agent. Their confounding model is a special case of our proposed Decoupled POMDP model in which confounding factors are independent of each other. \citet{lu2018deconfounding} propose a latent variable model for ``deconfounding reinforcement learning''. They extend the work of \citet{louizos2017causal} by positing a deep latent variable model with a single unobserved confounder that governs a trajectory,  deriving a variational lower bound on the likelihood and training a model with variational inference. Their causal model does not take into account dynamics of unobserved confounders. \citet{oberst2019counterfactual} also look at off-policy evaluation in POMDPs, though unlike this work they assume that the unobserved state does \emph{not} directly affect the observed behavior-policy actions. Their work focuses on counterfactuals: what would have happened in a specific trajectory under a different policy, had all the other variables, including the random noise variables, been the same. This is a difficult task, lying on the third rung of Pearl's causal hierarchy, which we restate in the supplementary material. \citep{pearl2018theoretical}. Our task is on the second rung of the hierarchy: we wish to know the effect of intervening on the world and acting differently, using a policy $\pi_e$. \citet{oberst2019counterfactual} therefore requires more stringent assumptions than ours on the structure of the causal model, namely an extension of outcome monotonicity.

Our work specifically extends the work of \citet{miao2018identifying}, which rests on the measurement of two independent proxy variables in a bandit setting. Our results generalizes their identification strategy through the independence structure that is inherent to POMDPs and Decoupled POMDPs, where past and future are independent conditioned on the unobserved confounder at time $t$.


\section{Conclusion and Future Work}

Off-policy evaluation of sequential decisions is a fundamentally hard problem, especially when it is done under partial observability of the state. Unknown states produce bias through factors that affect both observed actions and rewards. This paper offers one approach to tackle this problem in POMDPs and Decoupled POMDPs.

While the expressiveness of POMDPs is useful in many cases, it also comes with a substantial increase in complexity. Yet, one may not necessarily require the complete general framework to model complex problems. This paper takes a step towards an alternative model, Decoupled POMDP, for which unobserved factors are isolated, reducing OPE complexity, while maintaining the same expressive power as POMDPs. We note that Decoupled POMDPs may also benefit general purpose RL algorithms in partially observable environments.

In this work we experimented with a tabular environment. As future work, one may scale up to practical domains using latent space embeddings of the generalized weight matrices, as well as sophisticated sampling techniques that may reduce variance in approximation. 

\fontsize{9.2pt}{10.5pt} \selectfont
\bibliographystyle{aaai}
\bibliography{bibfile}

\section{Acknowledgments}
We thank Michael Oberst, Moshe Tennenholtz, and the anonymous reviewers for their fruitful comments that greatly improved this paper. Research was conducted under ISF grant number 1380/16.

\onecolumn
\fontsize{10.0pt}{11.0pt} \selectfont

\appendix
\section{Main Results}
\label{appendixA}

\subsection{Local Evaluation}

POMDPs are known to be notoriously hard to approximate \citep{madani1999undecidability,papadimitriou1987complexity}. Their intractability is mainly due to the ``curse of dimensionality'' for which complexity grows exponentially with the cardinality of the state space $\abs{\U}$. We thus first tackle a more modest problem, that is local in time, for which complexity does not scale with history length $L$.

Suppose we are tasked with the limited setting of evaluating $\pi_e$ only for a specific point in time, while behaving the same as the behavior policy $\pi_b$ at all other times. For illustrative purposes, we will restrict ourselves to an evaluation policy which only depends on the current observation $z$ (also known as a memoryless policy). This restriction will be removed in our final results and is only assumed for clarity. More specifically, given ${\pi_b^{(t)}: \U \times \A \mapsto [0, 1]}$ and ${\pi_e^{(t)} : \mathcal{Z} \times \A \mapsto [0, 1]}$ we define the time-dependent policy $\pi_L$ as
\begin{equation*}
    \pi_L^{(t)} =
    \begin{cases}
    \pi_e^{(t)} & t = L \\
    \pi_b^{(t)} & o.w.
    \end{cases}
\end{equation*}
Our goal is to evaluate $v_L(\pi_L)$ using the measure $P^{b}$ over observable histories. We generalize the bandit result presented in \citet{miao2018identifying}, in which two independent proxy variables of a hidden variable satisfying a certain rank condition are sufficient in order to nonparametrically evaluate $\pi_L$, even if the observation distribution $O(z,u)$ is unknown. We first note that for all times $t < L$, $P^{\pi_L}(r_t) = P^b(r_t)$, by definition. It is thus sufficient to evaluate $P^{\pi_L}(r_L)$. 

\noindent We have that
\begin{align}
    &P^{\pi_L}(r_L) \\
    &=
    \sum_{z_L, a_L, u_L} 
    P^{\pi_L}(r_L | z_L, a_L, u_L)
    P^{\pi_L}(z_L, a_L, u_L) \\
    &=
    \sum_{z_L, a_L, u_L}
    P^b(r_L | a_L, u_L)
    \pie{L}{a_L | z_L}
    P^b(z_L | u_L)
    P^b(u_L) \\
    &=
    \sum_{z_L, a_L, u_L} 
    \pie{L}{a_L | z_L}
    P^b(r_L, z_L | a_L, u_L)
    P^b(u_L),
    \label{eq: local effect}
\end{align}
where the last transition is due to the fact that $\pi_b$ does not depend on $z_L$ given $u_L$. Equation~\eqref{eq: local effect} can be rewritten in vector form as
\begin{equation*}
    P^{\pi_L}(r_L)
    =
    \sum_{z_L, a_L} 
    \pie{L}{a_L | z_L}
    P^b(r_L, z_L | a_L, U_L)
    P^b(U_L).
\end{equation*}
Next, note that $z_{L-1}$ is independent of $r_L$ and $z_L$ given $u_L$ and $a_L$. Therefore,
\begin{align*}
&P^b(r_L, z_L | a_L, U_L)P^b(U_L| a_L, z_{L-1}) \\
&= 
P^b(r_L, z_L | a_L, z_{L-1}, U_L)P^b(U_L| a_L, z_{L-1})  \\
&=
P^b(r_L, z_L | a_L, z_{L-1}).
\end{align*}
As the above is true for all values of $z_{L-1}$ and assuming the matrix $P^b(U_L | a_L, Z_{L-1})$ is invertible for all $a_L \in \A$, then we can write (in vector form)
\begin{equation}
P^b(r_L, z_L | a_L, U_L) 
=
P^b(r_L, z_L | a_L, Z_{L-1})
P^b(U_L| a_L, Z_{L-1})^{-1}.
\label{eq: local part1}
\end{equation}
Similarly, we have that
\begin{align*}
&P^b(z_L| a_L, U_L)
P^b(U_L | a_L, z_{L-1}) \\
&=
P^b(z_L| a_L, z_{L-1}, U_L)
P^b(U_L | a_L, z_{L-1}) \\
&=
P^b(z_L | a_L, z_{L-1}).
\end{align*}
As the above is true for all values of $z_{L-1}, z_L$ and assuming the matrix $P^b(Z_L | a_L, U_L)$ is invertible for all $a_L \in \A$, we can write
\begin{equation}
P^b(U_L | a_L, Z_{L-1})
=
P^b(Z_L| a_L, U_L)^{-1}P^b(Z_L | a_L, Z_{L-1}).
\label{eq: local part2}
\end{equation}
Combining Equations~\eqref{eq: local part1} and \eqref{eq: local part2} yields
\begin{align*}
    &P^b(r_L, z_L | a_L, U_L) \\
    &=
   P^b(r_L, z_L | a_L, Z_{L-1})
   P^b(Z_L | a_L, Z_{L-1})^{-1}
   P^b(Z_L| a_L, U_L).
\end{align*}
Finally, note that $z_L$ is independent of $a_L$ given $u_L$. Hence,
\begin{equation*}
    P^b(Z_L | a_L, U_L)P^b(U_L) = P^b(Z_L | U_L)P^b(U_L) = P^b(Z_L).
\end{equation*}
Lastly, note that if $P^b(Z_L| a_L, Z_{L-1})$ is invertible, then so is $P^b(Z_L | a_L, U_L)$. Denote the generalized weight vector $W_L(a_L) = P^b(Z_L | a_L, Z_{L-1})^{-1}P^b(Z_L)$, then we have thus proved the following proposition.

\begin{proposition}
\label{prop: local}
Suppose $P^b(Z_L | a_L, Z_{L-1}), P^b(U_L | a_L, Z_{L-1})$ are invertible for all $a_L \in \A$, then
\begin{align*}
    P^{\pi_L}(r_L) 
    =
    \sum_{z_L, a_L} 
    \pie{L}{a_L | z_L}
   P^b(r_L, z_L | a_L, Z_{L-1})
   W_L(a_L).
\end{align*}
\end{proposition}
Proposition~\ref{prop: local} lets us evaluate $P^{\pi_L}(r_L)$ without access to the unknown state $u_L$. It uses past and future observations $z_L$ and $z_{L-1}$ to create an unbiased evaluation of $P^{\pi_L}(r_L)$. Its main assumption is that the conditional distribution matrices $P^b(Z_L | a_L, Z_{L-1}), P^b(U_L | a_L, Z_{L-1})$ are invertible. In other words, it is assumed that enough information is transferred from states to observations between time steps. A trivial case in which this assumption does not hold is when $U_i$ are i.i.d. In such a scenario, $Z_{L-1}$ does not contain useful information in order to evaluate $r_L$, and an additional independent observation is needed. Nevertheless, this assumption can be greatly reduced under a decoupled POMDP model (see Section \ref{sec: Decoupled POMDP}).

At this point, we have all needed information to evaluate $v(\pi_L)$. The proof of Theorem~\ref{thm: main result} iteratively applies a similar dependence in order to evaluate $v(\pi_e)$ globally for all time steps and general history dependent evaluation policies.


\subsection{Proof of Theorem~\ref{thm: main result}}

We start by stating two auxiliary lemmas (their proof can be found in Appendix~B.

\begin{lemma}
    \begin{align*}
        P^e(r_t) 
        &=
        \sum_{\tau^o \in \T^o_t}
        \pth{\prod_{i=0}^t \pie{i}{a_i | h_i^o)}}
        P^b(r_t, z_t | a_t, U_t)
        \pth{
        \prod_{i=t-1}^{0}
        P^b(U_{i+1}, z_i | a_i, U_i)
        }
        P^b(U_0).
    \end{align*}
    \label{lemma: do_r}
\end{lemma}

\begin{lemma}
\label{lemma: multiplication}
    For all $0 \leq i \leq t-1$, let $x_i, y_i$ such that $x_i \perp (u_{i+1}, z_i) | a_i, u_i$ and $y_i \perp x_i | a_i, u_i$ and $y_i \perp (a_i, a_{i-1}, x_{i-1}, z_{i-1}) | u_i$. Assume that the matrices $P^b(U_i | a_i, X_i), P^b(Y_i | a_i, X_i), P^b(Y_i | a_i, U_i)$ are invertible. Then
    \begin{align*}
        P^b(Y_i| a_i, U_i)
        P^b(U_i, z_{i-1} | a_{i-1}, X_{i-1})
        =
        P^b(Y_i, z_{i-1} | a_{i-1}, X_{i-1}).
    \end{align*}
    Moreover,
    \begin{align*}
        &P^b(U_{i+1}, z_i | a_i, U_i)
        P^b(U_i, z_{i-1} | a_{i-1}, U_{i-1}) \\
        &=
        P^b(U_{i+1}, z_i | a_i, X_i)
        P^b(Y_i| a_i, X_i)^{-1}
        P^b(Y_i, z_{i-1} | a_{i-1}, X_{i-1})
        P^b(Y_{i-1}| a_{i-1}, X_{i-1})^{-1}
        P^b(Y_{i-1}| a_{i-1}, U_{i-1})).
        \end{align*}
    Additionally, let $x_t, y_t$ such that $x_t \perp (r_t, z_t) | a_t, u_t$ and $y_t \perp x_t | a_t, u_t$ and $y_t \perp (a_t, a_{t-1}, x_{t-1}, z_{t-1}) | u_t$. Assume that the matrices $P^b(U_t | a_t, X_t), P^b(X_t | a_t, U_t), P^b(Y_t | a_t, X_t)$ are invertible. Then
    \begin{align*}
        &P^b(r_t, z_t | a_t, U_t)
        P^b(U_t, z_{t-1} | a_{t-1}, U_{t-1}) \\
        &=
        P^b(r_t, z_t | a_t, X_t)
        P^b(Y_t| a_t, X_t)^{-1}
        P^b(Y_t, z_{t-1} | a_{t-1}, X_{t-1})
        P^b(Y_{t-1}| a_{t-1}, X_{t-1})^{-1}
        P^b(Y_{t-1}| a_{t-1}, U_{t-1})).
        \end{align*}
\end{lemma}

\noindent We are now ready to complete the proof of Theorem~\ref{thm: main result}. Using Lemmas~\ref{lemma: do_r} and \ref{lemma: multiplication}, it holds that
\begin{align*}
    &P^b(r_t, z_t | a_t, U_t)
    \pth{
    \prod_{i=0}^{t-1}
    P^b(U_{i+1}, z_i | a_i, U_i)
    }
    P^b(U_0) \\
    &=
    P^b(r_t, z_t | a_t, X_t)
    \pth{
    \prod_{i=t-1}^{0}
    P^b(Y_{i+1}| a_{i+1}, X_{i+1})^{-1}
    P^b(Y_{i+1}, z_i | a_i, X_i)
    }
    P^b(Y_0| a_0, X_0)^{-1}
    P^b(Y_0| a_0, U_0)
    P^b(U_0).
\end{align*}
As $y_0 \perp a_0 | u_0$
\begin{equation*}
    P^b(Y_0| a_0, U_0)
    P^b(U_0) =
    P^b(Y_0| U_0)
    P^b(U_0)
    =
    P^b(Y_0).
\end{equation*}
Hence,
\begin{align*}
    P^e(r_t) 
    &=
    \sum_{\tau^o \in \T^o_t}
    \pth{\prod_{i=0}^t \pie{i}{a_i | h_i^o)}}
    P^b(r_t, z_t | a_t, X_t)
    \pth{
    \prod_{i=t-1}^{0}
    P^b(Y_{i+1}| a_{i+1}, X_{i+1})^{-1}
    P^b(Y_{i+1}, z_i | a_i, X_i)
    }
    P^b(Y_0| a_0, X_0)^{-1}
    P^b(Y_0).
\end{align*} 
To complete the proof we let $x_i = z_{i-1}$ for $i \geq 1$ and $y_i = x_{i+1} = z_i$ for $i \geq 0$. Then independence assumptions of Lemma~\ref{lemma: multiplication} indeed hold. Moreover, it is enough to assume that $P^b(Z_i | a_i, Z_{i-1})$ is invertible for all $i \geq 1$ as
\begin{equation*}
    P^b(z_i | a_i, z_{i-1}) 
    = 
    \sum_{u_i} P^b(z_i | a_i, z_{i-1}, u_i)P^b(u_i | z_{i-1}, a_i).
\end{equation*}
Or in vector notation
\begin{equation*}
    P^b(Z_i | a_i, Z_{i-1}) = P^b(Z_i | a_i, U_i)P^b(U_i | Z_{i-1}, a_i).
\end{equation*}
Since $P^b(Z_i | a_i, Z_{i-1})$ is invertible, so are $P^b(U_i | Z_{i-1}, a_i) = P^b(U_i | X_i, a_i)$ and $P^b(Z_i | a_i, U_i) = P^b(Y_i | a_i, U_i)$. \\
This completes the proof of the theorem.

\subsection{Proof of Theorem~\ref{thm: main result2}}

We start by stating two auxiliary lemmas (their proof can be found in Appendix~B)

\begin{lemma}
\begin{align*}
        P^e(r_t) 
        &=
        \sum_{\tau^o \in \T^o_t}
        \pth{\prod_{i=0}^t \pie{i}{a_i | h_i^o)}}
        P^b(r_t, o_t | a_t, z_t, U_t)
        \pth{
        \prod_{i=t-1}^{0}
        P^b(U_{i+1}, z_{i+1}, o_i | a_i, z_i, U_i)
        }
        P^b(U_0 | z_0)P^b(z_0).
\end{align*}
    \label{lemma: do_r2}
\end{lemma}

\begin{lemma}
\label{lemma: multiplication2}
    For all $0 \leq i \leq t-1$, let $x_i, y_i$ such that $x_i$ is independent of $(u_{i+1}, o_i)$ given $z_{i+1}, a_i, u_i$, $y_i$ is independent of  $x_i$ given $z_{i+1}, a_i, u_i$, and $y_i$ is independent of $a_i, a_{i-1}, x_{i-1}, z_{i-1}$ given $u_i$. Assume that the matrices $P^b_{([n], J_i)}(U_i| a_i, z_i, X_i), P^b_{(K_i, [n])}(Y_i | a_i, z_i, U_i), P_{(K_i,J_i)}^b(Y_i | a_i, z_i, X_i)$ are invertible. Then \\
\begin{align*}
    P_{(K_i, [n])}^b(Y_i | a_i, z_i, U_i)
    P_{([n],J_{i-1})}^b(U_{i}, z_{i}, o_{i-1} | a_{i-1}, z_{i-1}, X_{i-1})
    =
    P^b_{(K_i, J_{i-1})}(Y_i, z_i, o_{i-1} | a_{i-1}, z_{i-1}, X_{i-1}).
\end{align*}
    Moreover,
\begin{align*}
    &P(U_{i+1}, z_{i+1}, o_i | a_i, z_i, U_i)
    P(U_{i}, z_{i}, o_{i-1} | a_{i-1}, z_{i-1}, U_{i-1}) \\
    &=
    P_{([n],J_i)}^b(U_{i+1}, z_{i+1}, o_i | a_i, z_i, X_i)
    P_{(K_i,J_i)}^b(Y_i | a_i, z_i, X_i)^{-1}
    P^b_{(K_i, J_{i-1})}(Y_i, z_i, o_{i-1} | a_{i-1}, z_{i-1}, X_{i-1}) \times \\
    &~~~~~ 
    P_{(K_{i-1},J_{i-1})}^b(Y_{i-1} | a_{i-1}, z_{i-1}, X_{i-1})^{-1}
    P_{(K_{i-1}, [n])}^b(Y_{i-1} | a_{i-1}, z_{i-1}, U_{i-1}).
\end{align*}

    Additionally, let $x_t, y_t$ such that $x_t$ is independent of $r_t, z_t$ given $a_t, u_t$, $y_t$ is independent of $x_t$ given $a_t, u_t$ and independent of $a_t, a_{t-1}, x_{t-1}, z_{t-1}$ given $u_t$. Assume that the matrices $P^b(U_t | a_t, X_t), P^b(X_t | a_t, U_t), P^b(Y_t | a_t, X_t)$ are invertible. Then
\begin{align*}
    &P(r_t, o_t | a_t, z_t, U_t)
    P(U_t, z_t, o_{t-1} | a_{t-1}, z_{t-1}, U_{t-1}) \\
    &=
    P_{J_t}^b(r_t, o_t | a_t, z_t, X_t)
    P_{(K_t,J_t)}^b(Y_t | a_t, z_t, X_t)^{-1}
    P^b_{(K_t, J_{t-1})}(Y_t, z_t, o_{t-1} | a_{t-1}, z_{t-1}, X_{t-1}) \times \\
    &~~~~~ 
    P_{(K_{t-1},J_{t-1})}^b(Y_{t-1} | a_{t-1}, z_{t-1}, X_{t-1})^{-1}
    P_{(K_{t-1}, [n])}^b(Y_{t-1} | a_{t-1}, z_{t-1}, U_{t-1}).
\end{align*}
\end{lemma}

\noindent We are now ready to complete the proof of Theorem~\ref{thm: main result2}. Using Lemmas~\ref{lemma: multiplication2} and \ref{lemma: do_r2}, it holds that
\begin{align*}
    &P^b(r_t, o_t | a_t, z_t, U_t)
    \pth{
    \prod_{i=t-1}^{0}
    P^b(U_{i+1}, z_{i+1}, o_i | a_i, z_i, U_i)
    }
    P^b(U_0 | z_0)P^b(z_0) \\
    &=
    P_{J_t}^b(r_t, o_t | a_t, z_t, X_t)
    \pth{
    \prod_{i=t-1}^{0}
    P_{(K_{i+1},J_{i+1})}^b(Y_{i+1} | a_{i+1}, z_{i+1}, X_{i+1})^{-1}
    P^b_{(K_{i+1}, J_i)}(Y_{i+1}, z_{i+1}, o_i | a_i, z_i, X_i)
    } \\
    &~~~
    P_{(K_0,J_0)}^b(Y_0 | a_0, z_0, X_0)^{-1}
    P^b_{(K_0, [n])}(Y_0 | a_0, z_0, U_0)
    P^b(U_0 | z_0)P^b(z_0)
\end{align*}
As $y_0 \perp a_0) | u_0, z_0$
\begin{equation*}
    P^b_{(K_0, [n])}(Y_0 | a_0, z_0, U_0)
    P^b(U_0 | z_0) 
    =
    P^b_{(K_0, [n])}(Y_0 | z_0, U_0)
    P^b(U_0 | z_0)
    =
    P^b(Y_0 | z_0).
\end{equation*}
To complete the proof we let $X_i = Z_{i-1}$ and $Y_i = O_i$ for $i \geq 0$. Then independence assumptions of Lemma~\ref{lemma: multiplication2} indeed hold. Moreover, it is enough to assume that $P^b(U_i | a_i, Z_{i-1})$ and $P^b(Z_i | a_i, Z_{i-1})$ are invertible for all $i \geq 1$ as
\begin{equation*}
    P^b(o_i | a_i, z_i, z_{i-1}) 
    = 
    \sum_{u_i} P^b(o_i | a_i, z_i, z_{i-1}, u_i)P^b(u_i | a_i, z_i, z_{i-1}).
\end{equation*}
Or in vector notation
\begin{equation*}
    P^b_{(K_i, J_i)}(O_i | a_i, z_i, Z_{i-1}) 
    = 
    P^b_{(K_i, [n])}(O_i | a_i, z_i, U_i) 
    P^b_{([n], J_i)}(U_i| a_i, z_i, Z_{i-1}).
\end{equation*}
Since $P^b_{(K_i, J_i)}(O_i | a_i, z_i, Z_{i-1}) $ is invertible, so are $P^b_{(K_i, [n])}(O_i | a_i, z_i, U_i) $ and $P^b_{([n], J_i)}(U_i| a_i, z_i, Z_{i-1})$. \\
This completes the proof of the theorem.

\newpage
\section{Auxilary Lemmas} \label{appendixB}

\subsection{Proof of Lemma \ref{lemma: do_r}}
\begin{proof}
\begin{align*}
    P^e(r_t) 
    &= 
    \sum_{\tau \in \T_t} 
    P^e(r_t | \tau)
    P^e(\tau) \\
    &= 
    \sum_{\tau \in \T_t} 
    P^{b}(r_t | a_t, u_t)
    P^e(\tau).
\end{align*}
Next we have that
\begin{align*}
    P^e(\tau) 
    &=
    P^e(u_0, z_0, a_0, \hdots, u_t, z_t, a_t) \\
    &=
    P^e(a_t | u_0, z_0, a_0, \hdots, u_t, z_t)
    P^e(u_0, z_0, a_0, \hdots, u_t, z_t) \\
    &=
    \pie{t}{a_t | h_t^o}
    P^e(z_t | u_0, z_0, a_0, \hdots, u_t)
    P^e(u_0, z_0, a_0, \hdots , u_{t-1}, z_{t-1}, a_{t-1}, u_t) \\
    &=
    \pie{t}{a_t | h_t^o}
    P^b(z_t | u_t)
    P^e(u_0, z_0, a_0, \hdots , u_{t-1}, z_{t-1}, a_{t-1}, u_t) \\
    &=
    \pie{t}{a_t | h_t^o}
    P^b(z_t | u_t)
    P^e(u_t | u_0, z_0, a_0, \hdots , u_{t-1}, z_{t-1}, a_{t-1})
    P^e(u_0, z_0, a_0, \hdots , u_{t-1}, z_{t-1}, a_{t-1}) \\
    &=
    \pie{t}{a_t | h_t^o}
    P^b(z_t | u_t)
    P^b(u_t | u_{t-1}, a_{t-1})
    P^e(u_0, z_0, a_0, \hdots , u_{t-1}, z_{t-1}, a_{t-1})
\end{align*}
By backwards induction we get that
\begin{align*}
    P^e(\tau) = 
    \pth{\prod_{i=0}^t \pie{i}{a_i | h_i^o} P^b(z_i | u_i)}
    \pth{\prod_{i=0}^{t-1} P^b({u_{i+1} | u_i, a_i}) }\nu_0(u_0).
\end{align*}
As $z_i$ is independent of $a_{i}, a_{i-1}$ given $u_i$ under measure $P^b$, we can write
\begin{align*}
   P^{b}(r_t | a_t, u_t)
   P^e(\tau)
   =
   P^{b}(r_t, z_t | a_t, u_t)
   \pth{\prod_{i=0}^t \pie{i}{a_i | h_i^o}}
    \pth{\prod_{i=0}^{t-1} P^b({u_{i+1}, z_i | u_i, a_i})}\nu_0(u_0),
\end{align*}
which in vector form yields
\begin{align*}
        P^e(r_t) 
        &=
        \sum_{\tau^o \in \T^o_t}
        \pth{\prod_{i=0}^t \pie{i}{a_i | h_i^o)}}
        P^b(r_t, z_t | a_t, U_t)
        \pth{
        \prod_{i=0}^{t-1}
        P^b(U_{i+1}, z_i | a_i, U_i)
        }
        P^b(U_0).
\end{align*}
Here, the summation has now changed to observable trajectories.
\end{proof}

\subsection{Proof of Lemma \ref{lemma: multiplication}}

We begin by proving an additional auxiliary lemma.

\begin{lemma}
\label{lemma: factorization}
Let $x, y, w, u, z, a$ be nodes in a POMDP model such that $x \perp (w,z) | a, u$, and $y \perp x | a, u$. Assume in addition that the matrices $P^b(U| a, X), P^b(Y | a, X), P^b(Y| a, U)$ are invertible for all $a$. Then
\begin{equation*}
    P^b(W, z | a, U)
    =
    P^b(W, z | a, X)P^b(Y | a, X)^{-1}P^b(Y | a, U).
\end{equation*}
\end{lemma}
\begin{proof}
We start by showing that
\begin{enumerate}
\item If $x \perp (w,z) | a, u$ and if $P^b(U| a, X)$ is invertible for every $a$, then
\begin{equation*}
    P^b(W, z | a, U) 
    =
    P^b(W, z | a, X)P^b(U| a, X)^{-1}.
\end{equation*}
 \item If $x \perp y | a, u$ and if $P^b(X | a, U)$ is invertible for every $a$, then
\begin{equation*}
P^b(U | a, X) 
=
P^b(Y | a, U)^{-1}P^b(Y | a, X).
\end{equation*}
\end{enumerate}
We have that
\begin{align*}
P^b(w, z | a, U)P^b(U| a, x) 
&= 
P^b(w, z | a, x, U)P^b(U| a, x)  \\
&=P^b(w, z | a, x)
\end{align*}
The above is true for every $\{x, w\}$, therefore
\begin{equation}
P^b(W, z | a, U) 
=
P^b(W, z | a, X)P^b(U| a, X)^{-1}.
\label{eq: lemma_part1}
\end{equation}
Similarly, for the second part, we have that
\begin{align*}
P^b(y| a, U)
P^b(U | a, x) 
&=
P^b(y | a, x, U)
P^b(U | a, x) \\
&=
P^b(y | a, x).
\end{align*}
The above is true for every $\{x, y\}$, therefore
\begin{equation}
P^b(U | a, X) 
=
P^b(Y | a, U)^{-1}P^b(Y | a, X).
\label{eq: lemma_part2}
\end{equation}
Combining Equations~\eqref{eq: lemma_part1} and \eqref{eq: lemma_part2} yields
\begin{equation*}
    P^b(W, z | a, U)
    =
    P^b(W, z | a, X)P^b(Y | a, X)^{-1}P^b(Y | a, U).
\end{equation*}
\end{proof}

\noindent We can now continue to prove Lemma~\ref{lemma: multiplication}.

\begin{proof}[Proof of Lemma~\ref{lemma: multiplication}]
Let $x_i, y_i$ such that $x_i \perp (u_{i+1}, z_i) | a_i, u_i$ and $y_i \perp x_i | a_i, u_i$. Assume that the matrices $P^b(U_i | a_i, X_i), P^b(X_i | a_i, U_i), P^b(Y_i | a_i, X_i), P^b(Y_i | a_i, U_i)$ are invertible. Then, by Lemma~\ref{lemma: factorization} 
\begin{equation*}
    P^b(U_{i+1}, z_i | a_i, U_i)
    =
    P^b(U_{i+1}, z_i | a_i, X_i)
    P^b(Y_i| a_i, X_i)^{-1}
    P^b(Y_i| a_i, U_i).
\end{equation*}
Next we wish to evaluate
\begin{align}
    &P^b(U_{i+1}, z_i | a_i, U_i)
    P^b(U_i, z_{i-1} | a_{i-1}, U_{i-1}) \\
    &=
    P^b(U_{i+1}, z_i | a_i, X_i)
    P^b(Y_i| a_i, X_i)^{-1}
    P^b(Y_i| a_i, U_i)
    P^b(U_i, z_{i-1} | a_{i-1}, X_{i-1})
    P^b(Y_{i-1}| a_{i-1}, X_{i-1})^{-1}
    P^b(Y_{i-1}| a_{i-1}, U_{i-1})).
    \label{eq: lemma4 result ref1}
\end{align}
For this, let us evaluate
\begin{equation*}
    P^b(Y_i| a_i, U_i)
    P^b(U_i, z_{i-1} | a_{i-1}, X_{i-1}).
\end{equation*}
Assume that $y_i \perp (a_i, a_{i-1}, x_{i-1}, z_{i-1}) | u_i$, then
\begin{align*}
&\sum_{u_i}P^b(y_i| a_i, u_i)
P^b(u_i, z_{i-1} | a_{i-1}, x_{i-1}) \\
&=
\sum_{u_i}P^b(y_i| a_{i-1}, z_{i-1}, x_{i-1}, u_i)
P^b(u_i, z_{i-1} | a_{i-1}, x_{i-1}) \\
&=
\sum_{u_i}P^b(y_i| a_{i-1}, z_{i-1}, x_{i-1}, u_i)
P^b(u_i | z_{i-1}, a_{i-1}, x_{i-1})
P^b(z_{i-1} | a_{i-1}, x_{i-1}) \\
&=
P^b(y_i| a_{i-1}, z_{i-1}, x_{i-1})
P^b(z_{i-1} | a_{i-1}, x_{i-1}) \\
&=
P^b(y_i, z_{i-1} | a_{i-1}, x_{i-1}).
\end{align*}
Therefore
\begin{align*}
P^b(Y_i| a_i, U_i)
P^b(U_i, z_{i-1} | a_{i-1}, X_{i-1})
=
P^b(Y_i, z_{i-1} | a_{i-1}, X_{i-1}).
\end{align*}
This proves the first part of the lemma. The second part immediately follows due to Equation~\eqref{eq: lemma4 result ref1}. That is,
\begin{align*}
    &P^b(U_{i+1}, z_i | a_i, U_i)
    P^b(U_i, z_{i-1} | a_{i-1}, U_{i-1}) \\
    &=
    P^b(U_{i+1}, z_i | a_i, X_i)
    P^b(Y_i| a_i, X_i)^{-1}
    P^b(Y_i, z_{i-1} | a_{i-1}, X_{i-1})
    P^b(Y_{i-1}| a_{i-1}, X_{i-1})^{-1}
    P^b(Y_{i-1}| a_{i-1}, U_{i-1})).
\end{align*}
As the above holds for all $i \geq 1$, the proof is complete. The proof for the third part follows the same steps with $u_{t+1}$ replaced by $r_t$.
\end{proof}

\newpage

\subsection{Proof of Lemma \ref{lemma: do_r2}}
\begin{proof}
\begin{align*}
    P^e(r_t) 
    &= 
    \sum_{\tau \in \T_t} 
    P^e(r_t | \tau)
    P^e(\tau) \\
    &= 
    \sum_{\tau \in \T_t} 
    P^{b}(r_t | a_t, z_t, u_t)
    P^e(\tau).
\end{align*}
Next we have that
\begin{align*}
    P^e(\tau) 
    &=
    P^e(u_0, z_0, o_0, a_0, \hdots, u_t, z_t, o_t, a_t) \\
    &=
    P^e(a_t | u_0, z_0, o_0, a_0, \hdots, u_t, z_t, o_t)
    P^e(u_0, z_0, o_0, a_0, \hdots, u_t, z_t, o_t) \\
    &=
    \pie{t}{a_t | h_t^o}
    P^b(o_t | u_0, z_0, o_0, a_0, \hdots, u_t, z_t)
    P^e(u_0, z_0, o_0, a_0, \hdots, u_t, z_t) \\
    &=
    \pie{t}{a_t | h_t^o}
    P^b(o_t | u_t)
    P^e(u_t, z_t | u_0, z_0, o_0, a_0, \hdots, u_{t-1}, z_{t-1}, o_{t-1}, a_{t-1})
    P^e(u_0, z_0, o_0, a_0, \hdots, u_{t-1}, z_{t-1}, o_{t-1}, a_{t-1}) \\
    &=
    \pie{t}{a_t | h_t^o}
    P^b(o_t | u_t)
    P^b(u_t, z_t | u_{t-1}, z_{t-1}, a_{t-1})
    P^e(u_0, z_0, o_0, a_0, \hdots, u_{t-1}, z_{t-1}, o_{t-1}, a_{t-1})
\end{align*}
By backwards induction we get that
\begin{align*}
    P^e(\tau) = 
    \pth{\prod_{i=0}^t \pie{i}{a_i | h_i^o} P^b(o_i | u_i)}
    \pth{\prod_{i=0}^{t-1} P^b({u_{i+1}, z_{i+1} | u_i, z_i, a_i}) }P^b(u_0 | z_0)P^b(z_0).
\end{align*}
As $o_i$ is independent of $a_i, z_i$ given $u_i$ under measure $P^b$, we can write
\begin{align*}
   P^{b}(r_t | a_t, z_t, u_t)
   P^e(\tau)
   =
   P^{b}(r_t, o_t | a_t, z_t, u_t)
   \pth{\prod_{i=0}^t \pie{i}{a_i | h_i^o}}
    \pth{\prod_{i=0}^{t-1} P^b({u_{i+1}, z_{i+1}, o_i | u_i, z_i, a_i}) }P^b(u_0 | z_0)P^b(z_0),
\end{align*}
which in vector form yields
\begin{align*}
        P^e(r_t) 
        &=
        \sum_{\tau^o \in \T^o_t}
        \pth{\prod_{i=0}^t \pie{i}{a_i | h_i^o)}}
        P^b(r_t, o_t | a_t, z_t, U_t)
        \pth{
        \prod_{i=t-1}^{0}
        P^b(U_{i+1}, z_{i+1}, o_i | a_i, z_i, U_i)
        }
        P^b(U_0 | z_0)P^b(z_0).
\end{align*}
Here, the summation has now changed to observable trajectories.
\end{proof}

\subsection{Proof of Lemma \ref{lemma: multiplication2}}

The proof of this lemma is similar to that of Lemma~\ref{lemma: multiplication} and is brought here for completeness. We begin by proving an auxiliary lemma.

\begin{lemma}
\label{lemma: factorization2}
Let $x, y, w, u, z, z', o, a$ be nodes in a Decoupled POMDP model such that $x \perp (w,o) | a, z, u$, and $y \perp x | a, z, u$. Let $I, J, K$ be index sets such that $\abs{I}=\abs{J}=\abs{K}=\abs{U}=n$. Also let $[n]$ be the index set $\{1, \hdots, |U|\}$. Assume in addition that the matrices $P^b_{[n], J}(U| a, z, X), P^b_{(K, [n])}(Y | a, z, U), P_{(K,J)}^b(Y | a, z, X)$ are invertible for all $a, z$. Then
\begin{equation*}
    P_{(I, [n])}^b(W, z', o | a, z, U)
    =
    P_{(I,J)}^b(W, z', o | a, z, X)P_{(K,J)}^b(Y | a, z, X)^{-1}P_{(K, [n])}^b(Y | a, z, U).
\end{equation*}
\end{lemma}
\begin{proof}
We start by showing that
\begin{enumerate}
\item If $x \perp (w, z', o) | a, z, u$ and if $P^b_{([n], J)}(U| a, z, X)$ is invertible for every $a, z$, then
\begin{equation*}
    P^b_{(I, [n])}(W, z', o | a, z, U) 
    =
    P^b_{(I, J)}(W, z', o | a, z, X)P^b_{[n], J}(U| a, z, X)^{-1}.
\end{equation*}
 \item If $x \perp y | a, z, u$ and if $P^b_{(J, [n])}(X | a, z, U)$ is invertible for every $a, z$, then
\begin{equation*}
P^b_{([n], J)}(U | a, z, X) 
=
P^b_{(K, [n])}(Y | a, z, U)^{-1}P^b_{(K, J)}(Y | a, z, X).
\end{equation*}
\end{enumerate}
We have that
\begin{align*}
P^b(w, z', o | a, z, U)P^b(U| a, z, x) 
&= 
P^b(w, z', o | a, z, x, U)P^b(U| a, z, x)  \\
&=P^b(w, z', o | a, z, x)
\end{align*}
The above is true for every $\{x, w\}$, therefore
\begin{equation}
P^b_{(I, [n])}(W, z', o | a, z, U)
=
P^b_{(I, J)}(W, z', o | a, z, X)P^b_{([n], J)}(U| a, z, X)^{-1}.
\label{eq: lemma_part1b}
\end{equation}
Similarly, for the second part, we have that
\begin{align*}
P^b(y| a, z, U)
P^b(U | a, z, x) 
&=
P^b(y | a, z, x, U)
P^b(U | a, z, x) \\
&=
P^b(y | a, z, x).
\end{align*}
The above is true for every $\{x, y\}$, therefore
\begin{equation}
P^b_{([n], J)}(U | a, z, X) 
=
P^b_{(K, [n])}(Y | a, z, U)^{-1}P^b_{(K, J)}(Y | a, z, X).
\label{eq: lemma_part2b}
\end{equation}
Combining Equations~\eqref{eq: lemma_part1b} and \eqref{eq: lemma_part2b} yields
\begin{equation*}
    P^b_{(I, [n])}(W, z', o | a, z, U)
    =
    P^b_{(I, J)}(W, z', o | a, z, X)P^b_{(K, J)}(Y | a, z, X)^{-1}P^b_{(K, [n])}(Y | a, z, U).
\end{equation*}
\end{proof}

\noindent We can now continue to prove Lemma~\ref{lemma: multiplication2}.

\begin{proof}[Proof of Lemma~\ref{lemma: multiplication2}]
Let $x_i, y_i$ such that $x_i \perp (u_{i+1}, o_i) | z_{i+1}, a_i, u_i$ and $y_i \perp x_i | z_{i+1}, a_i, u_i$. Assume that the matrices $P^b_{([n], J_i)}(U_i| a_i, z_i, X_i), P^b_{(K_i, [n])}(Y_i | a_i, z_i, U_i), P_{(K_i,J_i)}^b(Y_i | a_i, z_i, X_i)$ are invertible. Then, by Lemma~\ref{lemma: factorization2} 
\begin{equation*}
    P(U_{i+1}, z_{i+1}, o_i | a_i, z_i, U_i)
    =
    P_{([n],J_i)}^b(U_{i+1}, z_{i+1}, o_i | a_i, z_i, X_i)P_{(K_i,J_i)}^b(Y_i | a_i, z_i, X_i)^{-1}P_{(K_i, [n])}^b(Y_i | a_i, z_i, U_i).
\end{equation*}
Next we wish to evaluate
\begin{align}
    &P(U_{i+1}, z_{i+1}, o_i | a_i, z_i, U_i)
    P(U_{i}, z_{i}, o_{i-1} | a_{i-1}, z_{i-1}, U_{i-1}) \\
    &=
    P_{([n],J_i)}^b(U_{i+1}, z_{i+1}, o_i | a_i, z_i, X_i)
    P_{(K_i,J_i)}^b(Y_i | a_i, z_i, X_i)^{-1}
    P_{(K_i, [n])}^b(Y_i | a_i, z_i, U_i) \times \\
    &~~~~~ 
    P_{([n],J_{i-1})}^b(U_{i}, z_{i}, o_{i-1} | a_{i-1}, z_{i-1}, X_{i-1})
    P_{(K_{i-1},J_{i-1})}^b(Y_{i-1} | a_{i-1}, z_{i-1}, X_{i-1})^{-1}
    P_{(K_{i-1}, [n])}^b(Y_{i-1} | a_{i-1}, z_{i-1}, U_{i-1}).
    \label{eq: lemma4 result ref1}
\end{align}
For this, let us evaluate
\begin{equation*}
    P_{(K_i, [n])}^b(Y_i | a_i, z_i, U_i)
    P_{([n],J_{i-1})}^b(U_{i}, z_{i}, o_{i-1} | a_{i-1}, z_{i-1}, X_{i-1}).
\end{equation*}
Assume that $y_i \perp (a_i, a_{i-1}, x_{i-1}, z_{i-1}, z_i) | u_i$, then
\begin{align*}
&
\sum_{u_i \in \U}
P^b(y_i| a_i, z_i, u_i)
P^b(u_i, z_i, o_{i-1} | a_{i-1}, z_{i-1}, x_{i-1}) \\
&=
\sum_{u_i \in \U}
P^b(y_i| u_i, z_i, o_{i-1}, a_{i-1}, z_{i-1}, x_{i-1})
P^b(u_i, z_i, o_{i-1} | a_{i-1}, z_{i-1}, x_{i-1}) \\
&=
\sum_{u_i \in \U}
P^b(y_i, u_i, z_i, o_{i-1} | a_{i-1}, z_{i-1}, x_{i-1}) \\
&=
P^b(y_i, z_i, o_{i-1} | a_{i-1}, z_{i-1}, x_{i-1}).
\end{align*}
Therefore
\begin{align*}
P_{(K_i, [n])}^b(Y_i | a_i, z_i, U_i)
P_{([n],J_{i-1})}^b(U_{i}, z_{i}, o_{i-1} | a_{i-1}, z_{i-1}, X_{i-1})
=
P^b_{(K_i, J_{i-1})}(Y_i, z_i, o_{i-1} | a_{i-1}, z_{i-1}, X_{i-1}).
\end{align*}
This proves the first part of the lemma. The second part immediately follows due to Equation~\eqref{eq: lemma4 result ref1}. That is,
\begin{align*}
    &P(U_{i+1}, z_{i+1}, o_i | a_i, z_i, U_i)
    P(U_{i}, z_{i}, o_{i-1} | a_{i-1}, z_{i-1}, U_{i-1}) \\
    &=
    P_{([n],J_i)}^b(U_{i+1}, z_{i+1}, o_i | a_i, z_i, X_i)
    P_{(K_i,J_i)}^b(Y_i | a_i, z_i, X_i)^{-1}
    P^b_{(K_i, J_{i-1})}(Y_i, z_i, o_{i-1} | a_{i-1}, z_{i-1}, X_{i-1}) \times \\
    &~~~~~ 
    P_{(K_{i-1},J_{i-1})}^b(Y_{i-1} | a_{i-1}, z_{i-1}, X_{i-1})^{-1}
    P_{(K_{i-1}, [n])}^b(Y_{i-1} | a_{i-1}, z_{i-1}, U_{i-1}).
\end{align*}
As the above holds for all $i \geq 1$, the proof is complete. The proof for the third part follows the same steps with $(u_{t+1}, z_{t+1})$ replaced by $r_t$.
\end{proof}


\newpage
\section{Bridging the Gap between Reinforcement Learning and Causal Inference}

This section is devoted to addressing the definition and results of this paper in terminology common in the Causal Inference literature \citep{Pearl:2009:CMR:1642718,peters2017elements}. We begin with preliminaries on Structural Causal Models and Pearl's Do-Calculus, and continue by defining OPE as an identification problem for Causal Inference. Finally, we show that results presented in this paper relate to the counterfactual effect of applying the dynamic treatment $\pi_e$ \citep{hernan2019causality}.

\subsection{Preliminaries}

The basic semantical framework of our analysis relies on Structural Causal Models \citep{Pearl:2009:CMR:1642718}.
\begin{definition}[Structural Causal Models]
A structural causal model (SCM) $M$ is a 4-tuple $\langle U, V, F, P(U) \rangle$ where:
\begin{itemize}
    \item $U$ is a set of exogenous (unobserved) variables, which are determined by factors outside of the model.
    \item $V$ is a set $\{V_i\}_{i=1}^n$ of endogenous (observed) variables that are determined by variables in $U \cup V$.
    \item $F$ is a set of structural functions $\{f_i\}_{i=1}^n$, where each $f_i$ is a process by which $V_i$ is assigned a value ${v_i \gets f_i(pa^i, u^i)}$ in response to the current values of its parents $PA^i \subset V$ and ${U^i \subseteq U}$.
    \item $P(U)$ is a distribution over the exogenous variables $U$.
\end{itemize}
\end{definition}

Consider the causal graph of Figure~\ref{fig: Decoupled POMDP}a. This causal graph corresponds to an SCM that defines a complete data-generating processes $P^b$, which entails the observational distribution. It also defines the interventional distribution $P^e$: under $P^e$, the arrows labeled $\pi_e$ exist and denote a functional relationship between $a_t$ and $z_t$ given by the evaluation policy $\pi_e$, and the arrows labeled $\pi_b$ do not exist.

Queries are questions asked based on a specific SCM, and are often related to interventions, which can be thought of as idealized experiments, or as well-defined changes in the world. Formally, interventions take the form of fixing the value of one variable in an SCM and observing the result. The do-operator is used to indicate that an experiment explicitly modified a variable. Graphically, this blocks any causal factors that would otherwise affect that variable. Diagramatically, this erases all causal arrows pointing at the experimental variable. 
\begin{definition}[Interventional Distribution]~\\
Given an SCM, an intervention $I = \pearldo{x_i := \tilde{f}(\widetilde{PA}^i, \tilde{U}_i)}$ corresponds
to replacing the structural mechanism $f_i(PA^i, U^i)$ with $\tilde{f}_i(\widetilde{PA}^i, U^i)$. This includes the concept of atomic interventions, where we may write more simply $\pearldo{X_i = x}$.
\end{definition}
The interventional distribution is subsumed by the counterfactual distribution, which asks in retrospective what might have happened had we acted differently at the specific realization. Table~\ref{table: three layers} shows the 3-layer hierarchy of \citet{pearl2018theoretical}, together with the characteristic questions that can be answered at each level. While work such as \citet{oberst2019counterfactual} is centered around the counterfactual layer, this paper focuses on the interventional layer associated with the syntactic signature of the type $P(y|\pearldo{x}, z)$.

The do-calculus is the set of manipulations that are available to transform one expression into another, with the general goal of transforming expressions that contain the do-operator into expressions that do not contain it, and which involve only observable quantities. Expressions that do not contain the do-operator and include only observable quantities can be estimated from observational data alone, without the need for an experimental intervention. The do-calculus includes three rules for the transformation of conditional probability expressions involving the do-operator. They are stated formally below. Let $x, y, z, w$ be nodes in an SCM $G$. 
\begin{itemize}
\item \textbf{Rule 1 (Insertion/deletion of observations)}

$P(y|\pearldo{x},z,w)=P(y|\pearldo{x},w)$ if $y$ and $z$ are $d$-separated by $x\cup w$ in $G^*$, the graph obtained from $G$ by removing all arrows pointing into variables in $x$.

\item \textbf{Rule 2 (Action/observation exchange)}

$P(y|\pearldo{x},\pearldo{z},w)=P(y|\pearldo{x},z,w)$ if $y$ and $z$ are $d$-separated by $x \cup w$ in $G^\dagger$, the graph obtained from $G$ by removing all arrows pointing into variables in x and all arrows pointing out of variables in $z$.

\item \textbf{Rule 3 (Insertion/deletion of actions)}

$P(y|\pearldo{x},\pearldo{z},w)=P(y|\pearldo{x},w)$ if $y$ and $z$ are $d$-separated by $x \cup w$ in $G^\bullet$, the graph obtained from $G$ by first removing all the arrows pointing into variables in $x$ (thus creating $G^*$) and then removing all of the arrows pointing into variables in $z$ that are not ancestors of any variable in $w$ in $G^*$.
\end{itemize}

\newpage
\noindent In addition to the above, letting $\pi_e$ be a stochastic time-dependent evaluation policy, then under the SCM of Figure~\ref{fig: Decoupled POMDP}a we have the following lemma (we use $z_{0:t}$ to denote the set $\{z_0, \hdots z_t\}$).
\begin{lemma}
\label{lemma: do pi}
\begin{equation*}
    P(x_t | \pearldo{\pi}, z_{0:t})
    =
    \sum_{a_0, \hdots, a_t} 
    P\pth{x_t | \pearldo{a_0}, \hdots, \pearldo{a_t},  z_{0:t}}
    \prod_{i = 0}^t \pie{i}{a_i | h_i^o}.
\end{equation*}
\end{lemma}
\begin{proof}
We have that
\begin{equation*}
    P(x_t | \pearldo{\pi}, z_{0:t}) 
    =
    P\pth{x_t | \pearldo{\pi_e^{(0)}}, \hdots, \pearldo{\pi_e^{(t)}}, z_{0:t}},
\end{equation*}
where we recall that $\pi_e^{(i)}$ is the (possibly time-dependent) policy at time $i$.
It is enough to show that for all $0 \leq k \leq t$
\begin{align}
    &P(x_t | \pearldo{\pi_e^{(0)}}, \hdots, \pearldo{\pi_e^{(t)}}, z_{0:t})\\
    &= 
    \sum_{a_0, \hdots, a_k} 
    P\pth{x_t | \pearldo{a_0}, \hdots \pearldo{a_k}, \pearldo{\pi_e^{(k+1)}}, \hdots, \pearldo{\pi_e^{(t)}}, z_{0:t}} \prod_{i=0}^k \pie{i}{a_0 | h_i^o}.
\label{eq: induction}
\end{align}
Equation~\eqref{eq: induction} follows immediately by induction over $k$, as $\pi_e^{(i)}$ depends only on its previously observed history.
\end{proof}

\renewcommand*{\arraystretch}{1.3}
\begin{table}[t!]
\begin{center}
\begin{tabular}{ccc}
    \hline
    Level (Symbol) &
    Typical Activity &
    Typical Questions 
    \\
    \hline
    \multirow{3}{10em}{\centering 1. Association \\ $P(y|x)$} &
    \multirow{3}{10em}{\centering Seeing} &
    \multirow{3}{15em}{\centering What is? \\ How would seeing x \\ change my belief in y?} 
    \\ \\ \\
    \hline
    \multirow{2}{10em}{\centering *2. Intervention \\ $P(y|\pearldo{x}, z)$} &
    \multirow{2}{10em}{\centering Doing \\ Intervening} &
    \multirow{2}{15em}{\centering What if? \\ What if I do x?}  
    \\ \\
    \hline 
    \multirow{3}{10em}{\centering 3. Counterfactual \\ $P(y_x | x', y')$ } &
    \multirow{3}{10em}{\centering Imagining \\ Retrospection  } &
    \multirow{3}{15em}{\centering Why? \\ Was it x that caused y? \\ What if I had acted differently?} 
    \\ \\ \\
    \hline
\end{tabular}
\end{center}
\caption{The three layer causal hierarchy, as given in \protect\citet{pearl2018theoretical}. OPE in the form we give is part of the second, interventional layer. }
\label{table: three layers}
\end{table}

\subsection{OPE as an Identification Problem}
\noindent In order to define OPE as an intervention problem, we define the interventional value of a policy $\pi$ given an SCM as
\begin{equation*}
    v^{\pearldo{\pi}} = \E \pth{ R_L(\tau) ~\middle|~ \pearldo{\pi}}.
\end{equation*}
We are now ready to define OPE in POMDPs. Let $\pi_e, \pi_b$ be evaluation policies as defined in Definition~\ref{def: policy types}.
\\
\begin{adjustwidth}{120pt}{120pt}
\textit{The goal of off-policy evaluation in POMDPs is to evaluate $v^{\pearldo{\pi_e}}$ for the SCM of Figure~\ref{fig: Decoupled POMDP}a under the data generating process $P^b$}.
\end{adjustwidth}
~\\
We can now restate the results presented in the paper using the above terminology. We restate Theorem~\ref{thm: main result} and Lemma~\ref{lemma: do_r} below. An extension to Theorem~\ref{thm: main result2} and Lemma~\ref{lemma: do_r2} is done similarly.

\begin{theorem*}[POMDP Evaluation]
Assume $P^b(Z_i | a_i, Z_{i-1})$ and $P^b(U_i | a_i, Z_{i-1})$ are invertible for all $i$ and all $a_i \in \A$.
For any $\tau^o \in \T^o_t$ denote
 \begin{align*}
     &\Pi_e(\tau^o) = \prod_{i=0}^t \pie{i}{a_i | h_i^o}, ~
     &\Omega(\tau^o) = \prod_{i=0}^{t} W_{t-i}(\tau^o).
 \end{align*}
Then
\begin{equation*}
    P(r_t | \pearldo{\pi_e}) 
    =
    \sum_{\tau^o \in \T^o_t}
    \Pi_e(\tau^o)
    P^b(r_t, z_t | a_t, Z_{t-1})
    \Omega(\tau^o).
\end{equation*}
\end{theorem*}
The proof of the theorem follows the same steps as in Appendix~A. Nevertheless, it requires an alternate interpretation and proof of Lemma~\ref{lemma: do_r}, which we state and prove formally below.

\begin{lemma}
    \begin{align*}
        P(r_t | \pearldo{\pi_e}) 
        &=
        \sum_{\tau^o \in \T^o_t}
        \pth{\prod_{i=0}^t \pie{i}{a_i | h_i^o)}}
        P^b(r_t, z_t | a_t, U_t)
        \pth{
        \prod_{i=0}^{t-1}
        P^b(U_{i+1}, z_i | a_i, U_i)
        }
        P^b(U_0).
    \end{align*}
    \label{lemma: do_r causal}
\end{lemma}
\begin{proof}
    Let $\pi_e$ be an evaluation policy.
\begin{align}
    P(r_t | \pearldo{\pi_e}) 
    &= 
    \sum_{u_{0:t}} P(r_t | \pearldo{\pi_e}, u_{0:t}) P(u_{0:t} | \pearldo{\pi_e}) \\
    &=
    \sum_{u_{0:t}} 
    P(r_t | \pearldo{\pi_e}, u_{0:t}) 
    P(u_0 | \pearldo{\pi_e})
    \prod_{i=0}^{t-1} P(u_{i+1} | \pearldo{\pi_e}, u_{0:i})
    \label{eq: global intervention}
\end{align}
$P(u_0 | \pearldo{\pi_e}) = P^b(u_0) = \nu_0(u_0)$ by definition and rule 3. Next we evaluate $P(u_{i+1} | \pearldo{\pi_e}, u_{0:i})$, using the conditional independence relations of the POMDP causal graph  Figure~\ref{fig: Decoupled POMDP}(a).
\begin{align}
    &P(u_{i+1} | \pearldo{\pi_e}, u_{0:i}) = \\
    &~~~~~=
    \sum_{z_{0:i} \in \mathcal{Z}^{i+1}} 
    P(u_{i+1} | \pearldo{\pi_e}, z_{0:i}, u_{0:i})
    P(z_i | \pearldo{\pi_e}, z_{0:i-1}, u_{0:i})
    P(z_{0:i-1} | \pearldo{\pi_e}, u_{0:i}). \label{eq: transition factorization}
\end{align}
Next, by Lemma~\ref{lemma: do pi} and the rules of do-calculus we have in the POMDP causal graph:
\begin{align}
    P(u_{i+1} | \pearldo{\pi_e}, z_{0:i}, u_{0:i})
    &~~~=
    \sum_{a_{0:i} \in \A^{i+1}}
    P(u_{i+1} | \pearldo{a_0},\hdots, \pearldo{a_i}, z_{0:i}, u_{0:i}) 
    \prod_{k=0}^{i}\pie{k}{a_k | h^o_k} \\
    &\underset{(\text{rule }3)}{=}
    \sum_{a_{0:i} \in \A^{i+1}}
    P(u_{i+1} | \pearldo{a_i}, z_{0:i}, u_{0:i}) 
    \prod_{k=0}^{i}\pie{k}{a_k | h^o_k} \\
    &~~~=
    \sum_{a_i \in \A}
    P(u_{i+1} | \pearldo{a_i}, z_{0:i}, u_{0:i}) 
    \pie{i}{a_i | h^o_i} \\
    &\underset{(\text{rule }1)}{=}
    \sum_{a_i \in \A}
    P(u_{i+1} | \pearldo{a_i}, u_i) 
    \pie{i}{a_i | h^o_i} \\
    &\underset{(\text{rule }2)}{=}
    \sum_{a_i \in \A}
    P^b(u_{i+1} | a_i, u_i) 
    \pie{i}{a_i | h^o_i} 
    \label{eq: factor1}.
\end{align}
Also, using the POMDP causal graph and by Lemma~\ref{lemma: do pi},
\begin{align}
    P(z_i | \pearldo{\pi_e}, z_{0:i-1}, u_{0:i})
    &~~~=
    \sum_{a_{0:i-1} \in \A^i}
    P(z_i | \pearldo{a_0}, \hdots, \pearldo{a_{i-1}}, z_{0:i-1}, u_{0:i})
    \prod_{k=0}^{i-1}\pie{k}{a_k | h^o_k} \\
    &\underset{(\text{rule }3)}{=}
    \sum_{a_{0:i-1} \in \A^i}
    P^b(z_i | z_{0:i-1}, u_{0:i}) 
    \prod_{k=0}^{i-1}\pie{k}{a_k | h^o_k} \\
    &~~~=
    P^b(z_i | u_i).
    \label{eq: factor2}
\end{align}
Plugging Equations~\eqref{eq: factor1}, \eqref{eq: factor2} into \eqref{eq: transition factorization} yields
\begin{align}
    P(u_{i+1} | \pearldo{\pi_e}, u_{0:i})
    &=
    \sum_{z_{0:i} \in \Z^{i+1}} 
    P(u_{i+1} | \pearldo{\pi_e}, z_{0:i}, u_{0:i})
    P(z_i | \pearldo{\pi_e}, z_{0:i-1}, u_{0:i})
    P(z_{0:i-1} | \pearldo{\pi_e}, u_{0:i}) \\
    &=
    \sum_{z_{0:i} \in \Z^{i+1}} 
    \sum_{a_i \in \A}
    P^b(u_{i+1} | a_i, u_i) 
    \pie{i}{a_i | h^o_i}
    P^b(z_i | u_i)
    P(z_{0:i-1} | \pearldo{\pi_e}, u_{0:i}) \\
    &=
    \sum_{z_i \in \Z} 
    \sum_{a_i \in \A}
    P^b(u_{i+1} | a_i, u_i) 
    \pie{i}{a_i | h^o_i}
    P^b(z_i | u_i)
    \sum_{z_{0:i-1} \in \Z^{i}} 
    P(z_{0:i-1} | \pearldo{\pi_e}, u_{0:i}) \\
    &=
    \sum_{z_i \in \Z} 
    \sum_{a_i \in \A}
    P^b(u_{i+1} | a_i, u_i) 
    \pie{i}{a_i | h^o_i}
    P^b(z_i | u_i) \\
    &=
    \sum_{z_i \in \Z} 
    \sum_{a_i \in \A}
    P^b(u_{i+1} | a_i, z_i, u_i) 
    \pie{i}{a_i | h^o_i}
    P^b(z_i | a_i, u_i)
    \label{eq: u-do factorization}.
\end{align}

\noindent
We continue by evaluating $P(r_t | \pearldo{\pi_e}, u_{0:t})$. Similar to before,
\begin{align}
    &P(r_t | \pearldo{\pi_e}, u_{0:t}) = \\
    &\underset{(\text{rule }1)}{=}
    P(r_t | \pearldo{\pi_e}, u_t) \\
    &~~~=
    \sum_{z_t \in \Z} 
    P(r_t | \pearldo{\pi_e}, u_t, z_t)
    P(z_t | \pearldo{\pi_e}, u_t) \\
    &~~~=
    \sum_{z_t \in \Z} 
    \sum_{a_t \in \A}
    P^b(r_t | a_t, u_t)
    \pie{t}{a_t | h_t^o}
    P^b(z_t | u_t) \\
    &~~~=
    \sum_{z_t \in \Z} 
    \sum_{a_t \in \A}
    P^b(r_t | a_t, z_t, u_t)
    \pie{t}{a_t | h_t^o}
    P^b(z_t | a_t, u_t)
    \label{eq: r-do factorization}.
\end{align}
Plugging Equations~\eqref{eq: u-do factorization}, \eqref{eq: r-do factorization} into Equation \eqref{eq: global intervention} yields
\begin{align*}
    P(r_t | \pearldo{\pi_e}) 
    &=
    \sum_{u_{0:t}} 
    P(r_t | \pearldo{\pi_e}, u_{0:t}) 
    P(u_0 | \pearldo{\pi_e})
    \prod_{i=0}^{t-1} 
    P(u_{i+1} | \pearldo{\pi_e}, u_{0:i}) \\
    &=
    \sum_{\tau \in \T_t}
    P^{b}(r_t, z_t | a_t, u_t)
   \pth{\prod_{i=0}^t \pie{i}{a_i | h_i^o}}
    \pth{\prod_{i=0}^{t-1} P^b({u_{i+1}, z_i | u_i, a_i})}\nu_0(u_0),
\end{align*}
which can be written in vector form as
\begin{align*}
    P(r_t | \pearldo{\pi_e}) 
    &=
    \sum_{\tau^o \in \T_t^o}
    \pth{\prod_{i=0}^t \pie{i}{a_i | h^o_i}}
    P^b(r_t, z_t | a_t, U_t)
    \pth{
    \prod_{i=0}^{t-1}
    P^b(U_{i+1}, z_i | a_i, U_i)
    }
    P^b(U_0)
\end{align*}

\noindent This completes the proof.
\end{proof}

\newpage
\section*{Appendix D: Experimental Details}

\begin{figure*}[h!]
    \begin{center}
    \begin{subfigure}{0.3\linewidth}
    \includegraphics[width=\linewidth]{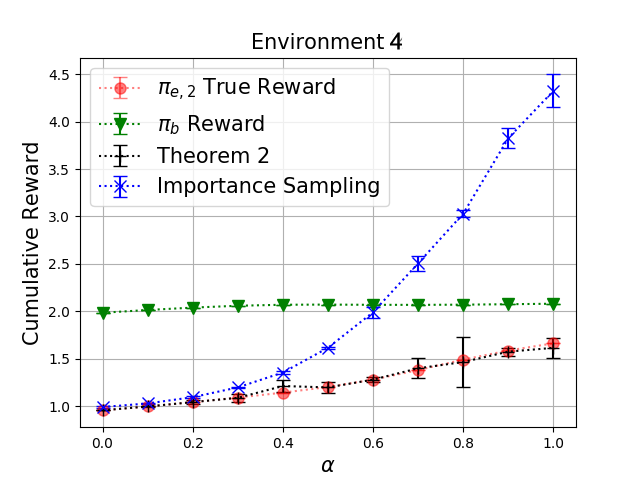}
    \end{subfigure}
    \begin{subfigure}{0.3\linewidth}
    \includegraphics[width=\linewidth]{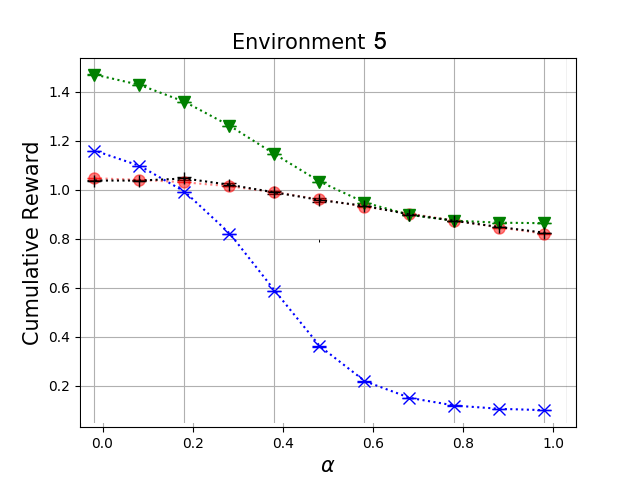}
    \end{subfigure}
    \label{fig: supp envs}
    \caption{}
    \end{center}
\end{figure*}

In our experiments we used three environments of different sampled vectors. Vectors were sampled uniformly from a normal distribution. We've chosen to depict in our paper, environments that had different characteristics of our results. Nevertheless, most sampled environments were unbiased for Theorem~\ref{thm: main result2}'s result, and highly biased for the IS estimator. Figure~4 depicts two other environments for which IS is highly biased. We also provide code of our experiments and medical environment in the supplementary material.


\end{document}